%% file: arxiv-seasoning.tex
\documentclass[11pt, a4paper, logo]{googledeepmind}

\usepackage[authoryear, sort&compress, round]{natbib}
\usepackage{upgreek,morenotations,rotating}

\title{\papertitle}

\correspondingauthor{hishamhusain@google.com}

\reportnumber{} 

\author[1]{Hisham Husain}
\author[2]{Valentin De Bortoli}
\author[1]{Richard Nock}

\affil[1]{Google Research}
\affil[2]{Google DeepMind}

\begin{abstract}
  \input{content/abstract}
\end{abstract}

\begin{document}

\maketitle

\input{content/introduction}
\input{content/preliminaries}
\input{content/theory}

\input{content/conclusion}

\bibliography{example_paper}

\newpage
\clearpage

\input{content/appendix}

\end{document}

%% file: content/abstract.tex
The use of discriminators to train or fine-tune generative models has proven to be a rather successful framework. A notable example is Generative Adversarial Networks (GANs) that minimize a loss incurred by training discriminators along with other paradigms that boost generative models via discriminators that satisfy weak learner constraints. More recently, even diffusion models have shown advantages with some kind of discriminator guidance. In this work, we extend a strong-duality result related to $f$-divergences which gives rise to a discriminator-guided recipe that allows us to \textit{refine} any generative model. We then show that the refined generative models provably improve generalization, compared to its non-refined counterpart. In particular, our analysis reveals that the gap in generalization is improved based on the Rademacher complexity of the discriminator set used for refinement. Our recipe subsumes a recently introduced score-based diffusion approach \cite{kim2022refining} that has shown great empirical success, however allows us to shed light on the generalization guarantees of this method by virtue of our analysis. Thus, our work provides a theoretical validation for existing work, suggests avenues for new algorithms, and contributes to our understanding of generalization in generative models at large.

%% file: content/introduction.tex
\section{Introduction}
In the past few years,  generative models \citep{hyvarinen2005estimation, vincent2011connection,song2019generative} have emerged as a popular method to efficiently learn distributions for several different domains such as images \citep{chung2022score, batzolis2021conditional, ruiz2023dreambooth}, text \citep{popov2021grad, kim2022diffusionclip}, graph \citep{fan2023generative, zhang2023survey} protein \cite{watson2023novo,abramson2024accurate,geffner2025proteina} or  audio data \citep{serra2022universal, pascual2023full, richter2023speech,wu2023duplex,qiang2023minimally}. Some examples of this include the Denoising Diffusion Probabilistic Models (DDPM) \citep{ho2020denoising}, along with the Denoising Diffusion Implicit Models (DDIM) \citep{song2020denoising} that have demonstrated success at a large scale such as DALL-E 2 \citep{ramesh2022hierarchical}.

Another very successful class of generative models, Generative Adversarial Networks (GANs) \citep{goodfellow2014generative, nowozin2016f}, learn distributions with the use of discriminators. It has been shown that learning this discriminator for GANs admits a (weakly) dual problem \citep{husain2019primal}, which corresponds to learning an encoder, coinciding with the objective derived in Wasserstein Autoencoders (WAE). However, the original motivation for discriminators in GANs is to guide a generative model on regions of under-performance with the use of a critic, i.e., the discriminator. Such advances motivated another line of work, referred to as \emph{boosted densities} \cite{cranko2019boosted, husain2020local, soen2020data}, that use discriminators as weak learners to boost generative models. The idea of utilizing discriminators has very recently been introduced into score-based generative models (SGMs), such as in \citep{kim2022refining}, which explicitly learns a density ratio and corrects the diffusion model. It was shown that this combination of discriminator and score-based model achieves a remarkable performance improvement. It is also well-known that adding discriminator, or classifier information, in the SGM denoising process yields better empirical results \cite{dhariwal2021diffusion}.

It has become clear that the utilization of discriminators, whether as for adversaries, weak learners or even to complement score-approximation, plays a pivotal role for generative models. In this work, we discover if there is a deeper connection between distribution estimation and binary classification, in order to cook a recipe to better train generative models. In order to do so, we study a well-known duality result for restricted variational $f$-divergences that takes the following form:
\begin{align}
    \sup_{h \in \mathcal{H}} \mathsf{R}(h) = \inf_{\mu} \mathsf{L}(\mu),
\end{align}
where $\mathcal{H}$ is the set of discriminators and $\mu$ is taken over all probability measures. This strong duality has been used in various applications to give insights into GAN objectives \cite{liu2018inductive, farnia2018convex} along with insights in autoencoder models \cite{husain2019primal}. In this work, we provide a non-trivial extension to this duality by deriving conditions under which we can construct the optimal $\mu^{*}$ from the optimal $h^{*}$ and find that it corresponds \textit{exactly} to the refined generative model constructed in discriminator-guided diffusion models \citep{kim2022refining}. 

This is rather striking, given that this link has not been explored previously, and allows us to propose a general recipe for refining generative models. The connection to strong duality allows us to prove an identity that quantifies the Integral Probability Metric (IPM) between the refined generative model and data in terms of the gap between the original generative model and data, along with an additional quantity that depends on the choice of discriminator set:
\begin{align}
    \operatorname{IPM}(\hat{P}, \mu_{\mathcal{H}}) = \mathsf{D}(\hat{P},\mu) - \mathsf{I}_{\mathcal{H}}(\mu, \mu_{\mathcal{H}}),
\end{align}
where $\hat{P}$ is the data, $\mu$ the original model and $\mu_{\mathcal{H}}$ the refined model.
The first term $\mathsf{D}(\hat{P},\mu)$ can be bounded by many existing lines of work such as in \citep{chen2022sampling,lee2022convergence,lee2023convergence,li2023towards,de2022convergence,li2024d} for score-based models and we provide an additional analysis on the benefits one gets from the second term, especially in the case when $\mathcal{H}$ contains discriminators that are more expressive. 

Finally, we establish a generalization bound that reveals additional insights into how refining a diffusion generative model in this way can close the generalization gap if discriminators $\mathcal{H}$ are well regularized. Our findings, therefore, advocate for discriminator refinement for deep generative models and validate existing work on discriminator refinement to a great deal of generality. In summary, our technical contributions come in three parts:\\
\noindent $\triangleright$ \textbf{(Theorem \ref{thm:dual-variables})} A characterization of strong duality for the $f$-divergence that reveals the optimality of refining generative models. This identity has interest outside the scope of this paper, as illustrated by the connections to variational inference \cite{knoblauch2019generalized} and boosted density estimation \cite{cranko2019boosted}.\\
\noindent $\triangleright$  An application to diffusion models, which shows how a general binary classifier played with a proper composite loss can be used to construct a refined diffusion model that admits theoretical convergence. We derive the choice of $f$ such that the refined diffusion corresponds to the framework of \citep{kim2022refining}.  \\
\noindent $\triangleright$ \textbf{(Theorem \ref{thm:generalization})} A study of generalization for refined diffusion models that unveils the importance of using discriminators to close the gap in generalization, along with guidance for the specific choices. This result parallels the discriminator-generalization trade-off other generative models, such as GANs, enjoy.

%% file: content/preliminaries.tex
\section{Related Work}
We split our related work into two sections; first, we focus on results focused on duality and discriminator studies in Generative Adversarial Networks (GANs), and then we turn to convergence and theory results for score-based diffusion models.

GANs were developed in \citep{goodfellow2014explaining} where a binary classification task was used to improve GANs, corresponding to Jensen-Shannon divergence minimization. This result was then generalized to $f$-divergences in \citep{nowozin2016f,nock2017f}. \citep{liu2017approximation} conducted the first convergence guarantees, showing that the learned distribution is indistinguishable from the data distribution under the chosen set of discriminators. \citep{liu2018inductive, zhang2017discrimination, husain2020distributional} then showed that the generalization abilities of GANs are related to the complexity of the discriminator set. In particular, using a discriminator set that is too large will yield poor generalization, whereas a set too small leads to under-discriminated models, hence a discrimination-generalization tradeoff.

Theoretical guarantees surrounding the convergence of score-based generative models (SGMs) largely consist of bounding the gap between the observed (finitely supported) data distribution and the distribution induced by the discretized diffusion model under various divergences between distributions. The majority of work assumes there exists a parametrized model that can approximate the true score function sufficiently well enough. \citep{song2020score} consider the Kullback-Leibler (KL) divergence and use Girsanov's Theorem to bound this quantity for the non-discretized diffusion model. \citep{lee2022convergence} then consider the discretized diffusion model and under log-Sobolev inequality (LSI) and smoothness of the true score provide convergence guarantees. \cite{chen2022sampling} and \cite{lee2023convergence} prove convergence without the LSI assumption. Another line of work assumes the score function and parametrized approximation is bounded at each point and proves convergence \citep{de2021diffusion}, with improved results under the manifold assumption \citep{de2022convergence}.

 Focusing on the Kullback-Leibler divergence, linear upper bounds with respect to the dimension were established in  \cite{benton2024nearlydlinearconvergencebounds,conforti2025diffusion}, where the authors obtained bounds of the form $O(d/\varepsilon^2)$ where $\varepsilon$ is the $\ell_2$ error on the score.  \cite{Li2024-em,Li2024-rn} improve this to $O(d/\varepsilon)$. At the time of submission, the best rates of convergence are  $O(\min(d, d^{2/3}L^{1/3}, d^{1/3}L) \varepsilon^{-2/3})$ and were obtained by \cite{Jiao2024-si} leveraging a relaxed Lipschitz constant. 

None of these results, however, consider discriminator-intervened diffusion models such as in \citep{kim2022refining}. Our work shows that the Integral Probability Metric (IPM) distance between the data and the refined diffusion model is equal to the discrepancy between the original generative model and data, which any of the above results can bound. Therefore, our result extends and builds upon existing convergence guarantees since we will show our framework can recover the instantiation in \citep{kim2022refining}.

\section{Preliminaries}
\label{sec:preliminaries}
\paragraph{Notation}
We use $\Omega$ to denote a compact Polish space and denote $\Sigma$ as the standard Borel $\sigma$-algebra on $\Omega$, $\mathbb{R}$  denotes the real numbers and $\mathbb{N}$ natural numbers. We use $\mathscr{F}(\Omega, \mathbb{R})$ to denote the set of all bounded and measurable functions mapping from $\Omega$ into $\mathbb{R}$ with respect to $\Sigma$, $\mathscr{B}(\Omega)$ to be the set of finite signed measures and the set $\mathscr{P}(\Omega) \subset \mathscr{B}(\Omega)$ will denote the set of probability measures. For any random variable $\X$, we use $\mathcal{L}(\X)$ to denote the law of $\X$. Let $\mathcal{N}(\mu,\Sigma)$ denote the $d$-dimensional Gaussian distribution with mean $\mu \in \mathbb{R}^d$ and covariance matrix $\Sigma$. For any proposition $\mathscr{I}$, the Iverson bracket is $\llbracket \mathscr{I} \rrbracket = 1$ if $\mathscr{I}$ is true and $0$ otherwise. We say a set of functions $\mathcal{H}$ is convex if $\lambda h + (1 - \lambda) h' \in \mathcal{H}$ for all $h,h' \in \mathcal{H}$ and $\lambda \in [0,1]$. For a function $h \in \mathscr{F}(\Omega, \mathbb{R})$ and metric $c: \Omega \times \Omega \to \mathbb{R}$, the Lipschitz constant of $h$ (w.r.t $c$) is $\operatorname{Lip}_c(h) = \sup_{\omega, \omega' \in \Omega} \card{h(\omega) - h(\omega')}/c(\omega, \omega')$ and $\nrm{h}_{\infty} := \sup_{\omega \in \Omega} \card{h(\omega)}$.  For any set of functions $\mathcal{H} \subseteq \mathscr{F}(\Omega, \mathbb{R})$, we use $\chull{\mathcal{H}}$ to denote the closed convex hull of $\mathcal{H}$ and use $\nrm{\mathcal{H}} = \sup_{h \in \mathcal{H}} \nrm{h}_{\infty}$ as the maximum bound for all functions in $\mathcal{H}$.

\paragraph{Generative Adversarial Losses}
We begin this section by first defining two important divergences between distributions: the $f$-divergence and Integration Probability Metrics (IPMs). Let $f: \mathbb{R} \to (-\infty, \infty]$ be a convex lower semi-continuous function such that $f(1) = 0$. The $f$-divergence between two probability measures $\mu,\nu \in \mathscr{P}(\Omega)$ is defined as $\mathsf{I}_f(\mu : \nu) := \E_{X \sim \nu}\left[f\left(\frac{d\mu}{d\nu}(\X)\right)\right]$, if $\mu \ll \nu$ (existence of Radon-Nikodym derivative) otherwise $\mathsf{I}_f(\mu : \nu) = + \infty$. The $f$ function is often referred to as the \textit{generator} of $\mathsf{I}_f$. On the other hand, for a given set of functions $\mathcal{H} \subseteq \mathscr{F}(\Omega,\mathbb{R})$, the IPM is defined as
\begin{align}
    d_{\mathcal{H}}(\mu,\nu) = \sup_{h \in \mathcal{H}} \braces{\E_{\X \sim \mu}[h(\X)] - \E_{\X \sim \nu}[h(\X)] }
\end{align}
The general adversarial loss we consider is that adapted from the generative adversarial networks literature \citep{liu2017approximation, zhang2017discrimination}. Let $f: \mathbb{R} \to (-\infty, \infty]$ be a convex lower semi-continuous function such that $f(1) = 0$ and let $\mathcal{H} \subseteq \mathscr{F}(\Omega, \mathbb{R})$ denote a set of bounded and measurable functions. We define a distance between distributions $\mu,\nu \in \mathscr{P}(\Omega)$ to be of the form
\begin{align}
    \mathsf{D}_{f,\mathcal{H}}(\nu,\mu) := \sup_{h \in \mathcal{H}}\braces{\E_{\X \sim \nu}[h(\X)] - \E_{\X \sim \mu}[f^{\star} \circ h(\X)] },\label{varDfH}
\end{align}
where $f^{\star}(t) = \sup_{t' \in \operatorname{dom} f} \bracket{t \cdot t' - f(t')}$ is the Fenchel conjugate of $f$. Note that $\mathsf{D}_{f,\mathcal{H}}$ can be viewed as a weaker divergences compared to $f$-divergences and IPMs noting that $f^{\star}(t) \geq t$ and so 
\begin{align}
    \mathsf{D}_{f,\mathcal{H}} \leq d_{\mathcal{H}} \mbox{ and } \mathsf{D}_{f,\mathcal{H}} \leq \mathsf{I}_f. \label{boundDdI}
\end{align}
Moreover, if $\mathcal{H}$ is chosen large enough then $\mathsf{D}_{f,\mathcal{H}}$ coincides with $\mathsf{I}_f$ for any optimal $h^* \in \mathcal{H}$ in \eqref{varDfH}. The divergence $\mathsf{D}_{f,\mathcal{H}}$ has been of great interest in the GAN community most notably appearing as the $f$-GAN objective \citep{nowozin2016f}, along with other theoretical studies such as in \cite{liu2017approximation,liu2018inductive,husain2019primal,husain2020distributional}.

The divergence $\mathsf{I}_f$ is tied to binary classification in Savage's theory of properness \citep{sEO}. Let $\mathcal{Y} = \braces{-1,1}$ and define a joint distribution $\mathbb{P}(\X,\Y)$ such that $\Y \in \mathcal{Y}$ with $\mathbb{P}(\X \mid \Y = -1) = \mu(\X)$, $\mathbb{P}(\X \mid \Y = +1) = \nu(\X)$ and $\mathbb{P}(\Y = -1) = \mathbb{P}(\Y = 1) = 1/2$. We then define a loss function $\ell_f: \mathcal{Y} \times [0,1] \to \mathbb{R}$, with an invertible link $\Psi: (0,1) \to \mathbb{R}$ set to be $\Psi(z) := f'(z/(1-z))$:
\begin{align}
    &\ell_f(+1,z) = -f'\bracket{\frac{\Psi^{-1}(z)}{1 - \Psi^{-1}(z)}}\\ &\ell_f(-1,z) = f^{\star}\bracket{f'\bracket{\frac{\Psi^{-1}(z)}{1 - \Psi^{-1}(z)}} }.\label{defPlosses}
\end{align}
Each of those two functions is called a partial loss \citep{rwCB,rwID}. We then have the following connection (to simplify our discussion later, we assume that $\mathcal{H}$ is rich enough, see \citet{nock2017f})
\begin{align}
    \inf_{h \in \mathcal{H}} \E_{(\X,\Y) \sim \mathbb{P}} \left[\ell_{f}\bracket{\Y,h(\X)}\right] = -\frac{1}{2} \cdot \mathsf{I}_f(\mu,\nu). \label{optBAYES}
\end{align}
Losses defined in \eqref{defPlosses} are not any kind of losses: they are proper composite \citep{nock2017f}, \textit{i.e.} they elicit Bayes prediction as an optimal predictor, composite meaning using $\Psi$ to link real-valued prediction to class probabilities. In words, if $\mathcal{H}$ is chosen large enough then $\mathsf{I}_f$ is also proportional to the loss of Bayes rule.

%% file: content/theory.tex
\section{How to season Generative Models}
For a given pre-trained generative model $\mu \in \mathscr{P}(\Omega)$ and a discriminator set $\mathcal{H}$, we will propose a method to refine $\mu$ into a new generative model $\mu^{\mathcal{H}}$ that has provable guarantees on generalization. In the first section, we will collect ingredients from strong duality by extending a particular result related to $f$-divergences, which has additional benefits beyond our purpose of improving the generalization of generative models. In the next section, we present the core recipe and discuss specific instantiations of this algorithm in score-based discriminator guided models, along with new frameworks. Finally, in the last section, we prove theoretical guarantees for this recipes such as generalization bounds. 

\subsection{Ingredients from strong-duality}
Our proposed recipe is based on an extended strong-duality result. 
Before we present the recipe, we will first revisit the primal-dual relationship in GANs and prove a stronger result in the context of denoising diffusion models. Let $\nu \in \mathscr{P}(\Omega)$ denote the data distribution we are interested in learning. First we recall the duality result \cite{liu2018inductive, husain2019primal} where if $\mathcal{H}$ is convex and closed under additive constants then
\begin{align}\label{eq:gan-duality}
    \mathsf{D}_{f,\mathcal{H}}(\nu,\mu) = \inf_{\overline{\mu} \in \mathscr{P}(\Omega)} \bracket{d_{\mathcal{H}}(\nu, \overline{\mu}) + \mathsf{I}_f(\overline{\mu} : \mu) }.
\end{align}
Under various settings, the optimization over $\overline{\mu}$ was shown to coincide with the optimization over an encoder function in \cite{husain2019primal} where the $d_{\mathcal{H}}$ is a reconstruction loss, and $\mathsf{I}_f$ is a regularizer however there is no characterization of $\overline{\mu}$ beyond this. We now present the first result in this direction.
\begin{theorem}
\label{thm:dual-variables}
    Let $f: \mathbb{R} \to (-\infty,\infty]$ be a strictly convex lower semi-continuous and differentiable function with $f(1) = 0$ and let $\mathcal{H}$ be convex and closed under addition. Denote by $h^{*} \in \mathcal{H}$ as the optimal solution of the primal problem in \eqref{eq:gan-duality}. Assuming $f'^{-1}\bracket{t } \geq 0$ for all $t \in \operatorname{dom}(f^{\star})$, consider a distribution $\mu^{\mathcal{H}}$ whose Radon-Nikodym derivative with respect to $\mu$ is
    \begin{align}\label{eq:mu-defn}
        \frac{d\mu^{\mathcal{H}}}{d\mu}(\cdot) = f'^{-1}\bracket{h^{*}(\cdot) }.
    \end{align}
    Then $\mu^{\mathcal{H}}$ is an optimal solution to \eqref{eq:gan-duality}.
\end{theorem}
\begin{proof}(Sketch, full proof in Appendix) 
    We note that the main strong duality result of interest surrounds the equation $\mathscr{L}(h,Q) = \E_{\nu}[h] - \E_{Q}[h] + \mathsf{I}_f(Q :\mu)$. We first show that if we fix $h$ and minimize $Q$ then the optimal $\mu^h$ satisfies $\frac{d\mu_h}{d\mu} = f'^{-1}(h - \lambda_h)$ for some normalizing constant $\lambda_h$. Next, we show that the optimal $h^{*}$ under the assumptions of $\mathcal{H}$ will be such that $\lambda_{h^{*}} = 0$ and thus by combining results in strong duality, we get that the optimal $\mu_{h^{*}}$ satisfies the required form above.
\end{proof}
When the distribution $\mu$ admits a Lebesgue density, we can write $\mu^{\mathcal{H}} = \mu \cdot f'^{-1}(h^*)$. While we aim to derive a general recipe for improving generative models, it boasts a more considerable generality beyond this application. For example, in a somewhat trivialized setting where $\mathcal{H}$ is selected to be minimal, this Theorem recovers results of Variational Inference, which are of independent interest.
\begin{example}[Variational Inference]
In the setting of $f(t) = t\log t$, the duality in \eqref{eq:gan-duality} recovers the Evidence Lower Bound (ELBO) that commonly appears in Variational Inference (VI) \citep{knoblauch2019generalized, husain2022adversarial}. To see this, we introduce a data-dependent loss function $L: \Theta \to \mathbb{R}$ where $\Theta$ is the parameter space of models, then we construct $\mathcal{H} = \braces{-L + b : b \in \mathbb{R}}$ as the minimal set satisfying convexity and closure under addition. Noting that $f^{\star}(t) = \exp(t-1)$ for the choice of KL-divergence, the duality then becomes
\begin{align}
    &\sup_{b \in \mathbb{R}} \bracket{\E_{\nu}[-L] + b - \E_{\mu}[\exp\bracket{-L + b - 1}] }\\ &= \inf_{\overline{\mu}} \bracket{ \E_{\nu}[-L] - \E_{\overline{\mu}}[-L] + \operatorname{KL}(\overline{\mu},\nu) }\\
    &\implies \sup_{b \in \mathbb{R}} \bracket{b - \E_{\mu}[\exp\bracket{-L + b - 1}] }\\ &= \inf_{\overline{\mu}} \bracket{\E_{\overline{\mu}}[L] + \operatorname{KL}(\overline{\mu},\nu) }\\
    &\implies \log \E_{\mu} \left[\exp\bracket{-L} \right]\\ &= \inf_{\overline{\mu}} \bracket{\E_{\overline{\mu}}[L] + \operatorname{KL}(\overline{\mu},\nu) },
\end{align}
which is precisely the ELBO as presented in VI where $\overline{\mu}$ is referred to as the posterior and $\mu$ is the prior distribution. When applying Theorem \ref{thm:dual-variables}, we note that $h^{*} = -L - \log\bracket{\E_{\mu}\left[\exp\bracket{-L-1} \right] }$ (see Appendix for more details) then the optimal distribution $\mu^{\mathcal{H}}$ as per Theorem \ref{thm:dual-variables} is
\begin{align}
    \mu^{\mathcal{H}} = &\exp\bracket{-L - \log\bracket{\E_{\mu}\left[\exp\bracket{-L-1} \right]} - 1} \cdot \mu\\
    &= \mu \cdot \exp\bracket{-L}  / \E_{\mu}[-L],
\end{align}
which is exactly the generalized Bayesian posterior.
\end{example}

\subsection{A general recipe for refining generative models}
While in the previous section, we extended a strong duality result for $f$-divergences, we will now demonstrate an algorithmic implication. First note that a direct consequence of Theorem \ref{thm:dual-variables} is the following
\begin{align} 
     d_{\mathcal{H}}\bracket{\nu,\mu^{\mathcal{H}}} = \mathsf{D}_{f,\mathcal{H}}\bracket{\nu,\mu} - \mathsf{I}_{f}\bracket{\mu^{\mathcal{H}} : \mu}. \label{thm:main-identity}
 \end{align}
 This means that if we have a pretrained generative model $\mu$ whose goal is to be close to some data distribution $\nu$, then $\mu^{\mathcal{H}}$ will be closer to $\nu$ (in $\mathcal{H}$ IPM) by a decrement of $\mathsf{I}_{f}\bracket{\mu^{\mathcal{H}} : \mu}$. The question now becomes, how does one access $\mu^{\mathcal{H}}$? In this case, we are interested in only sampling from $\mu^{\mathcal{H}}$ which is our final goal in generative models. In order to do so, we can use the score function of $\mu^{\mathcal{H}}$ by virtue of Theorem \ref{thm:dual-variables}:
\begin{align}
    \nabla \log \mu^{\mathcal{H}} = \nabla \log \mu + \nabla \log\bracket{f'^{-1}\bracket{h^{*}(\cdot) } }.
\end{align}
Conveniently, if we have access to the score function of $\mu$, then we add an extra term that involves a transformation of $h^{*}$. Algorithmically, accessing $\mu^{\mathcal{H}}$ thus involves two steps:
\begin{enumerate}
    \item Find the optimal $h^{*}$ by solving \begin{align}
    h^{*} = \argsup_{h \in \mathcal{H}}\bracket{\E_{\X \sim \nu}[h(\X)] - \E_{\X \sim \mu}[f^{\star} \circ h(\X)] }.
    \end{align}
    \item Sample from $\mu^{\mathcal{H}}$ using the score function $\nabla \log \mu + \nabla \log\bracket{f'^{-1}\bracket{h^{*}(\cdot) } }.$
    
\end{enumerate}

Moreover, in order to sample from $\mu^{\mathcal{H}}$, one can derive the score function, which by virtue of the above Theorem is
\begin{align}
\label{eq:guidance_practical}
    \nabla \log \mu^{\mathcal{H}} = \nabla \log \mu + \nabla \log\bracket{f'^{-1}\bracket{h^{*}(\cdot) } }.
\end{align}
In denoising diffusion models, we approximate the score function $\nabla \log \mu$ using a parametrized neural network $s_{\upvartheta}$ and thus we can \textit{refine} the diffusion model by using a score function $s_{\upvartheta} + \log\bracket{f'^{-1}\bracket{h^{*}(\cdot) } }$ instead. Furthermore, note that solving for $h^{*}$ is equivalent to binary classification with a particular loss function derived by $f$, label classes $\nu$, $\mu$ and hypothesis class $\mathcal{H}$ (see Section \ref{sec:preliminaries}). We thus refer to this refined distribution $\mu^{\mathcal{H}}$ as \textit{discriminator-guided}. 

Note that here, we only consider the gradient of the current generative model and data distribution, whereas in diffusion model we would consider \emph{noisy} estimates of those. We would leverage \eqref{eq:guidance_practical} for every noise level. More precisely,
\begin{align}
\label{eq:guidance_pratical_time}
    \nabla \log \mu^{\mathcal{H}}_t = \nabla \log \mu_t + \nabla \log\bracket{f'^{-1}\bracket{h^{*}_t(\cdot) } },
\end{align}
where $\mu_t$, $\nu_t$ are noisy version of $\mu$ and $\nu$ respectively and $h^*_t$ the optimal discriminator between $\nu_t$ and $\mu_t$. it is well-known that following the backward process \eqref{eq:guidance_pratical_time} does not yield samples from the distribution $\mu^{\mathcal{H}}$. However, there exists several techniques based on Feynamn-Kac techniques to correct this behavior \cite{thornton2025composition, skreta2025feynman, singhal2025general}. 

We will now show two examples that our two-step recipe above recovers two paradigms used to learn generative models. 

\begin{example}[Boosted Density Estimation \citep{cranko2019boosted}]
If we select $f(t) = t \log t$, then $\mu^{\mathcal{H}}$ simplifies to
\begin{align}
    \mu^{\mathcal{H}} = \mu \cdot \exp(h^{*}) / \E_{\mu}[h^{*}],\label{eq:bde}
\end{align}
which corresponds to an exponential family whose base measure is $\mu$ and sufficient statistic $h^{*}$ (the cumulant is defined from the denominator). It should be noted that a line of related work \citep{cranko2019boosted, husain2020local} that boost and improve densities via binary classifiers using the form of \eqref{eq:bde}. In particular, the work of \cite{soen2020data} not only takes the form of \label{eq:bde} but also samples from $\mu^{\mathcal{H}}$ using the augmented score function.
\end{example}
We now move onto the case where picking $f$ corresponds to $h^{*}$ being the cross-entropy minimized discriminator, deriving exactly the framework of \cite{kim2022refining}, which is a recently proposed method to efficiently refine diffusion models.
\begin{example}[Discriminator-Guided Diffusion Models \citep{kim2022refining}]
    In the setting of $f(t) = t\log t - (t+1) \log(t+1) + 2\log 2$, if we define $\mathcal{H} = \braces{\log \eta_{\theta}  : \eta_{\theta} : \Omega \to [0,1], \theta \in \Theta }$, where $\eta_{\theta}$ is a parametrized model giving a softmax score such a neural network. The duality then becomes
    \begin{align}
        &\mathsf{D}_{f,\mathcal{H}}(\nu,\mu) =\\ &2 \log 2 -  \inf_{\theta \in \Theta} \bracket{\E_{\nu}[-\log\bracket{\eta_{\theta}}] + \E_{\mu}[-\log\bracket{1 - \eta_{\theta} }] } , \label{eq:cross-ent}
    \end{align}
    which corresponds to binary cross-entropy minimization. Considering that $f'^{-1}(t) = \frac{\exp t}{1 - \exp t}$, the refined distribution $\mu^{\mathcal{H}}$ simplifies to
    \begin{align}
        \frac{d\mu^{\mathcal{H}}}{d\mu} = \frac{\eta_{\theta^{*}}}{1 - \eta_{\theta^{*}}},
    \end{align}
    where $\theta^{*}$ is the optimal parameter from \eqref{eq:cross-ent}.
\end{example}
In the work of \cite{kim2022refining}, it was found that by first learning a score function $s_{\upvartheta}$ at each iteration of diffusion followed by a refinement term that corresponds exactly to that which we have derived. While it has an intuitive argument, the connection to this strong duality result is non-trivial and rather striking.

\subsection{Generalization bounds for refined generative models}
We now study the generalization properties of $\mu^{\mathcal{H}}$ and a make an important observation on its relation to $\mathcal{H}$. In order to proceed, we first define an important quantity that often counters the discriminative abilities of $\mathcal{H}$: the Rademacher complexity. Let $R$ denote the uniform distribution over $\braces{-1,+1}$ then for a class of functions $\mathcal{H} \subseteq \mathscr{F}(\Omega, \mathbb{R})$ and distribution $P$, the \textit{Rademacher} complexity \citep{bartlett2002rademacher} is
\begin{align}
    \mathscr{R}_n(\mathcal{H}) := \E_{\zeta \sim R^n, \X \sim P^n}\left[\sup_{h \in \mathcal{H}} \frac{1}{n} \sum_{i=1}^n \zeta_i h(\X_i)\right],
\end{align}
where $R^n$ and $P^n$ are distributions over $n$-tuples. The Rademacher complexity is a cornerstone in many generalization studies, specifically dominating supervised learning; however, it has also appeared in unsupervised learning domains and generative models such as in \citep{zhang2017discrimination}. We refer the reader to \citep{liang2016cs229t} for a more comprehensive analysis of this quantity for different choices of model complexity. We now link the Rademacher complexity of our discriminator set used for refinement $\mathcal{H}$ to the gap between the refined diffusion model and data distribution $P$.
\begin{theorem}\label{thm:generalization}
    Let $f: \mathbb{R} \to (-\infty, \infty]$ be a strictly convex lower semi-continuous and differentiable function with $f(1) = 0$. Suppose $f'^{-1}(t) \geq 0$ for all $t \in \operatorname{dom}(f^{\star})$ and let $P$ denote the population distribution of $\hat{P}$ such that each $x_i \sim P$ i.i.d. It then holds that
    \begin{eqnarray}
        d_{\mathcal{H}}\bracket{P, \mu^{\mathcal{H}}} &\leq &\mathsf{D}_{f,\mathcal{H}}\bracket{\hat{P}, \mu} - \mathsf{I}_f\bracket{\mu^{\mathcal{H}} : \mu} + \mathscr{R}_n(\mathcal{H}) + 2 \nrm{\mathcal{H}} \sqrt{\frac{1}{2n} \ln\bracket{\frac{1}{\delta}}},\label{eqB}
    \end{eqnarray}
    with probability at least $1 - \delta$.
\end{theorem}
\begin{proof}(Sketch, full proof in Appendix)
    We apply the main duality strong from Equation \eqref{thm:main-identity} and apply concentration inequalities to the term $\mathsf{D}_{f,\mathcal{H}}(P, \mu)$ which allows us to connect it to $\mathcal{D}_{f,\mathcal{H}}(\hat{P},\mu)$ and the Rademacher complexity of $\mathcal{H}$, concluding the proof.
\end{proof}

We now provide a detailed analysis of the key terms in the RHS of \eqref{eqB}, from left to right, except for the fourth, rightmost term, which is just "the" ordinary slow rate statistical penalty.

\paragraph{First term} The first term in the upper bound can be upper bounded by both $\mathsf{D}_f(\hat{P},\mu)$ and $d_{\mathcal{H}}(\hat{P},\mu)$ which are the natural objectives that one minimizes when first learning $\mu$. For example, in score-based diffusion models, this quantity is bounded for the reverse-KL divergence ($f(t) = -\log t$) \citep{lee2023convergence} or Total Variation ($f(t) = \card{t-1}$) \cite{chen2022sampling}, both showing fast rates of convergence with respect to $n$. Alternately, we can upper bound via the IPMs noting that $d_{\mathcal{H}} \leq \sup_{h \in \mathcal{H}} \operatorname{Lip}(h)$, where $\operatorname{Lip}(h)$ is the Lipschitz constant of $h$ and $W_1$ is the $1$-Wasserstein distance. In this setting, \cite{oko2023diffusion} shows a score-based diffusion model $\mu$ can approximate $\hat{P}$ in $W_1$ at a minimax optimal rate. 

\paragraph{Second term via reparametrization} We now present an analysis on the second term, which indicates the gain from $\mu$ to $\mu^{\mathcal{H}}$. In order to better understand how positive this quantity can be, we reparametrize it into a $1$-dimensional integral: suppose the classifier decomposes with the natural composite link: $h^{*} = f'(\eta_{h^{*}} / (1 - \eta_{h^{*}}))$ where $\eta_{h^{*}}: \Omega \to [0,1]$ is the optimal class probability estimate.

\begin{theorem}
\label{thm:gain}
    Let $f: \mathbb{R} \to (-\infty, \infty]$ be a strictly convex lower semi-continuous and differentiable function with $f(1) = 0$. Denote by $h^{*} \in \mathcal{H}$ as the optimal discriminator and $\mu^{\mathcal{H}}$ from Theorem \ref{thm:dual-variables}. If $f'^{-1}(t) \geq 0$ for all $t \in \operatorname{dom}(f^{\star})$ then we have
    \begin{align}
        \mathsf{I}_f\bracket{\mu^{\mathcal{H}} : \mu} = \int_{0}^1 f\bracket{\frac{t}{1 - t}}d\rho_{h^{*}}(t),
    \end{align}
where $\rho_{h^{*}} := {\eta_{h^{*}}}_{\#} \mu$ and $\eta_{h^{*}} := f'^{-1}(h^{*})/(1+f'^{-1}(h^{*}))$.
\end{theorem}

Note here that $\rho_{h^{*}}$ corresponds to the distribution over $[0,1]$ where the classifier $h^{*}$ predicts $\mu$ to be, and $\#$ is the push-forward operator. Let us consider the case of binary-cross entropy corresponding to \citep{kim2022refining} where $f(t) =  t \log t - (t+1) \log(t+1) + 2 \log 2$ gives Jensen-Shannon divergence. The expression then simplifies to  
\begin{align}
    \int_0^1  \bracket{\frac{t}{1-t} \log t + \log(1 - t) + 2\log 2}d\rho_{h^{*}}(t).
\end{align}
Recalling that $\eta_{h^{*}}(\X)$ corresponds to the class probability estimation of a point $\X \in \Omega$ belonging to the class $\hat{P}$ (as opposed to $\mu$), we can say that if $\mathcal{H}$ is rich enough to classify between $\hat{P}$ and $\mu$ then $\eta_{h^{*}}$ is close to $0$ on the support of $\mu$ and thus $\rho_{h^{*}}(t)$ is concentrated around $0$, thereby increasing the above expression to $f(0) = 2\log 2$, the maximal possible value for $\mathsf{I}_f\bracket{\mu^{\mathcal{H}} : \mu}$. On the other hand, if the best classifier $\eta_{h^{*}}$ is unable to classify the two classes, then $\eta_{h^{*}} \to 1/2$ (a random classifier) and thus the above expression tends to $f(1) = 0$. Therefore, the gain in refinement can be understood by the discriminative abilities of $\mathcal{H}$, and a richer choice for $\mathcal{H}$ will lead to more improvement. 

\paragraph{Second term via Savage's theory of properness} The second term being the most important, we provide a second analysis via \citet{sEO} and the partial losses already defined in \eqref{defPlosses}. Fix $\Psi(z) := z$ for simplicity so that those partial losses are parameterized by $t\in [0,1]$ instead of $z \in \mathbb{R}$. The \textit{pointwise} loss defined by \eqref{defPlosses} is
\begin{align}
L_f(\eta, t) := \eta \ell_f(1,t) + (1-\eta) \ell_f(-1,t).
\end{align}
An indication of the inherent difficulty of the prediction task, which is orthogonal to learning because it represents the hardness of guessing $\mu$ vs $\mu^{\mathcal{H}}$, is the minimal value achievable, called pointwise Bayes loss, $\underline{L}_f(\eta) := \inf_t L_f(\eta, t)$ \citep{rwCB, rwID}. 

\begin{lemma}
\label{lem:prop}
For the partial losses defined in \eqref{defPlosses}, it holds that
\begin{align}
\underline{L}_f(\eta) & = - (1-\eta)\cdot f\left(\frac{\eta}{1-\eta}\right). \label{linkLf}
\end{align}
\end{lemma}
\begin{proof}
When the loss is proper, which is the case for the partial losses defined in \eqref{defPlosses}, the inf value is obtained at $t=\eta$ \citep{rwID,sEO} so that
\begin{align*}
\underline{L}_f(\eta) & = \eta \ell_f(+1,\eta) + (1-\eta) \ell_f(-1,\eta),
\end{align*}
and we can simplify the expression of $\ell_f(-1,t)$ in \eqref{defPlosses} for $f$ strictly convex differentiable\footnote{The convex conjugate satisfies $f^\star(z) = z f'^{-1}(z) - f(f'^{-1}(z))$.}:
\begin{align*}
\ell_f(-1,t) & = \frac{t}{1-t}\cdot f'\left(\frac{t}{1-t}\right) - f\left(\frac{t}{1-t}\right)\\
& = - \frac{t}{1-t}\cdot \ell_f(+1,t) - f\left(\frac{t}{1-t}\right),
\end{align*}
and yields after multiplying by $1-t$ and reorganizing \eqref{linkLf}, as claimed.
\end{proof}
Importantly, \eqref{linkLf} is (negative) a generalized perspective transform of $f$ \citep{mOAI,mOAII,nmoAS}, which is a general class of transformation of a convex function. Crucially, it preserves convexity (it also follows directly from the fact that $\underline{L}_f$ is concave, \citet{rwID}). In the reparametrization section, we discuss the case of \citep{kim2022refining}, which is in fact very "easy" from the standpoint of properness: its corresponding proper loss is the log-loss, which is symmetric ($\ell_f(+1,\eta) := -\log(\eta) = \ell_f(-1,1-\eta)$). As a consequence, $\underline{L}_f(\eta)$ is symmetric around $\eta = 1/2$ and gets us the maximum loss, or the minimal lowerbound for $I_f$ \eqref{optBAYES}, which it turns out matches $f(1) = 0$. Properness tells us more broadly that such lowerbounds can be attained anywhere in $(0,1)$ for specific choices of $f$ -- we indeed have $\underline{L}'_f = \ell_f(+1,t) - \ell_f(-1,t)$ and $\underline{L}_f$ is concave so the max happens when the partial losses coincide. The fact that maximum difficulty can happen at $\eta \neq 1/2$ does not contradict the intuitions behind our analysis of \citep{kim2022refining}; it happens naturally when we consider the prediction problem from its broadest perspective: Bayes posterior $\eta$ is in fact proportional to the \textit{prior} \citep{rwID}, $\eta \propto \pi$ -- the probability that the label be +1 -- which is usually implicitly taken to be 1/2 out of any knowledge to fix it accurately enough otherwise, and justifies the use of symmetric losses. From a Bayesian standpoint, the maximum loss for Bayes posterior has to happen when it brings no additional knowledge compared to the prior $\pi$ alone, because the inherent difficulty of the prediction problem is maximal, i.e. $\mu^{\mathcal{H}} = \mu$. $\pi$ is hidden in the definition of $\mathsf{I}_f\bracket{\mu^{\mathcal{H}} : \mu}$ but if we slide it from $1/2$ to anywhere in $(0,1)$, it amounts to using the perspective transform of $f$ \eqref{linkLf} to craft a whole set of different $f$s (see Remark \ref{rem-prior} below), susceptible to increase the quality of our bound via an increase of $\mathsf{I}_f\bracket{\mu^{\mathcal{H}} : \mu}$. The standpoint of properness in the second term is thus to show how a fundamental normative theory of classification can improve our bound in \eqref{eqB} via a simple analytical lever on the design of a transformation of a convex function that preserves convexity. We believe more convexity-preserving transformations are available with a sound classification basis, that would bring further concrete cases for increased $\mathsf{I}_f\bracket{\mu^{\mathcal{H}} : \mu}$.

\begin{remark}\label{rem-prior} We note that $f$ also intervenes in $\mathsf{D}_{f,\mathcal{H}}$ but fortunately it is a lowerbound of $\mathsf{I}_f$ \eqref{boundDdI} and it relies on the convex conjugate of $f$. To see the impact of sliding the prior $\pi$ on the convex conjugate, we can reverse compute $f$ from some pointwise Bayes loss $\underline{L}$ and prior $\pi$ \citep[Theorem 9]{rwID}:
\begin{align}
f_\pi (u) &= - (1-\pi + \pi u)\cdot \underline{L}\left( \frac{\pi u}{1-\pi + \pi u}\right).
\end{align}
Note the rather non-trivial way the prior influences the design of $f$. We then derive that the convex conjugate satisfies
\begin{align*}
f^\star_\pi (t) &= \frac{1-\pi}{\pi}\cdot \sup_{z\in [0,1)} \frac{tz + \pi(1-\pi) \cdot \underline{L}(z)}{1-z}\\
& > \frac{1-\pi}{\pi}\cdot \sup_{z\in [0,1)} tz + \pi(1-\pi) \cdot \underline{L}(z)\\
& \quad = (1-\pi)^2 \cdot (-\underline{L})^\star \left(\frac{z}{\pi(1-\pi)}\right),
\end{align*}
and this lowerbound is $((1-\pi)/\pi)$-Lipschitz. Thus, the impact on $\mathsf{D}_{f,\mathcal{H}}$ is moderate and can also yield a decrease of it.
\end{remark}

\paragraph{Third term} Notably, the third term of the bound is the Rademacher complexity of $\mathcal{H}$, which prefers a more restricted choice, thus forming a trade-off in generalization. We therefore find that the decision of selecting $\mathcal{H}$ parallels the story in statistical learning theory \cite{bartlett1998sample} where there is a trade-off between ensuring $\mathcal{H}$ is expressive enough and the Rademacher complexity is small enough to ensure fast convergence. In connection to \cite{kim2022refining}, our work advocates for the regularization of disriminators $\mathcal{H}$ as opposed to having them be as expressive as possible. Our results suggest that by using deep neural networks that are able to generalize well and be expressive on the training data, give added benefits to this discriminator-guided framework.

%% file: content/conclusion.tex
\section{Conclusion}
We provide a recipe to improve existing generative models with the use of a discriminator set $\mathcal{H}$. In order to prove our result, we revisit the primal-dual link in GANs and extend results in that direction, characterizing the exact distribution under which strong duality holds, showing that it corresponds to the refinement in \cite{kim2022refining}. This result allows us to quantify exactly the Integral Probability Metric (IPM) between the data and the proposed refined generative model distribution, leading us to characterize the generalization abilities of such models. In particular, our results advocate using regularized discriminators and refining diffusion models to improve generalization.

Some ways to extend this work include a tighter analysis for the term $\mathsf{D}_{f,\mathcal{H}}\bracket{\hat{P},\mu}$ which we bounded by either the IPM or $f$-divergence to utilize existing results. This term, however, is smaller and weaker than both these divergences; therefore, a tighter analysis can reveal a faster convergence rate. Additionally, since our framework recovers the binary cross entropy refinement from \cite{kim2022refining}, we can derive loss functions beyond this case. Since our goal was to prove the link to strong duality, we leave the implementation of refined diffusion models with discriminators learned by minimizing general proper losses as the subject for future work.

%% file: content/appendix.tex
\section{Appendix}
\subsection{Notation}
We will be invoking general convex analysis on the space $\mathscr{F}(\Omega, \mathbb{R})$, in the same fashion as \cite{liu2018inductive}, noting that $\mathscr{F}(\Omega, \mathbb{R})$ is a Hausdorff locally convex space (through the uniform norm). We use $\mathscr{B}(\Omega)$ to denote the denote the set of all bounded and finitely additive signed measures over $\Omega$ (with a given $\sigma$-algebra). For any set $D \subseteq \mathscr{B}(\Omega)$ and $h \in \mathscr{F}(\Omega, \mathbb{R})$, we use $\sigma_{D}(h) = \sup_{\nu \in D} \ip{h}{\nu}$ and $\delta_{D}(\nu) = \infty \cdot \llbracket \nu \notin D \rrbracket$ to denote the \textit{support} and \textit{indicator} functions such as in \cite{rockafellar1970convex}. We introduce the conjugate specific to these spaces
\begin{definition}[\cite{rockafellar1968integrals}]
For any proper convex function $F: \mathscr{F}(\Omega, \mathbb{R}) \to (-\infty, \infty)$, we have for any $\mu \in \mathscr{B}(\Omega)$ we define
\begin{align*}
    F^{\star}(\mu) =\sup_{h \in \mathscr{F}(\Omega, \mathbb{R})} \bracket{ \int_{\Omega} h d\mu - F(h) }
\end{align*}
and for any $h \in \mathscr{F}(\Omega, \mathbb{R})$ we define
\begin{align*}
    F^{\star \star}(h) = \sup_{\mu \in \mathscr{B}(\Omega)} \bracket{ \int_{\Omega} h d\mu - F^{\star}(\mu) }.
\end{align*}
\end{definition}
\begin{theorem}[\cite{zalinescu2002convex} Theorem 2.3.3]
\label{self-conjugacy}
If $X$ is a Hausdorff locally convex space, and $F: X \to (-\infty, \infty]$ is a proper convex lower semi-continuous function then $F^{\star \star} = F$.
\end{theorem}
We also define the Frechet normal cone (or prenormal cone) of a prescribed set $\mathcal{H} \subseteq \mathscr{F}(\Omega, \mathbb{R})$ is
\begin{align}
    \mathsf{N}_{\mathcal{H}}(h) := \braces{\mu' \in \mathscr{B}(\Omega) : \limsup_{\mathcal{H} \subseteq (h') \to h} \frac{\ip{\mu'}{h' - h}}{\nrm{h' - h}} \leq 0 }, 
\end{align}
where the $\ip{\cdot}{\cdot}$ is the operator linking the two dual spaces. in this case of $\mathscr{F}(\Omega,\mathbb{R})$ and $\mathscr{B}(\Omega)$ is $\ip{h}{\mu} = \E_{\X \sim \mu}[h(\X)]$
\subsection{Proof of Theorem \ref{thm:dual-variables}}
We begin by defining the objective
\begin{align}
    \mathscr{L}(h,Q) = \E_{\nu}[h] - \E_{Q}[h] + \mathsf{I}_f(Q :\mu).
\end{align}
Note that since $\mathscr{L}$ is upper semicontinuous concave in $h$ and lower semicontinuous convex in $Q$, we have by Ky Fan's minimax Theorem \citep{fan1953minimax} and \cite[Lemma~27]{liu2018inductive}:
\begin{align}
    \sup_{h \in \mathcal{H}} \inf_{Q \in \mathscr{P}(\Omega)} \mathscr{L}(h,Q) =  \inf_{Q \in \mathscr{P}(\Omega)} \sup_{h \in \mathcal{H}} \mathscr{L}(h,Q).
\end{align}
\begin{lemma}\label{lem:opt-mu}
    For any function $h \in \mathscr{F}(\Omega,\mathbb{R})$, we have
    \begin{align}
        \mathscr{L}(h,\mu_h) = \inf_{Q \in \mathscr{P}(\Omega)} \mathscr{L}(h,Q),
    \end{align}
    where $\mu_h$ is a probability measure such that the Radon-Nikodym derivative satisfies \begin{align}
        \frac{d\mu_h}{d\mu} = f'^{-1}(h - \lambda_h),
    \end{align}
    where $\lambda_h$ is a constant such that $\E_{\mu}[f'^{-1}(h - \lambda_h)] = 1$.
\end{lemma}
\begin{proof}
    First note $Q$ must be absolutely continuous with respect to $\mu$ since it is a requirement of $\mathsf{I}_f$ being finite. We can therefore re-parametrize $Q$ as $Q = r \cdot \mu$ where $\E_{\mu}[r] = 1$. We also require $r \geq 0$ however for now, we will show our assumptions on $f$ can alleviate this. The optimization can thus be rewritten as
\begin{align}
    &\E_{\nu}[h] - \E_{Q}[h] + \mathsf{I}_f(Q : \mu)\\
    &= \E_{\nu}[h] - \E_{\mu}\left[ r \cdot h \right]+ \E_{\mu}[f(r)] + \lambda (\E_{\mu}[r] - 1)  \\
    &= \E_{\nu}[h] + \E_{\mu}\left[- r \cdot h + f(r) + \lambda (r - 1) \right] .
\end{align}
We differentiate the objective w.r.t $r$, and set the derivative to $0$, which yields:
\begin{align}
    0 = -h^{*} + f'(r) + \lambda\\
    \implies r = f'^{-1}\bracket{h^{*} - \lambda}.
\end{align}
Differentiating with respect to $\lambda$ yields $\E_{\mu}[r] = 1$ which can be satisfied if we set $\lambda = \overline{\lambda}$. Next, note that the assumption $f'^{-1}(t) \geq 0$ guarantees this solution satisfies $r \geq 0$. Thus, the optimal density satisfies
\begin{align}
    \frac{d\mu_h}{d\mu} = f'^{-1}(h - \lambda_h).
\end{align}
\end{proof}
Note that
\begin{align}
    \mathscr{L}(h,\mu_h) &= \E_{\nu}[h] - \E_{Q}[h] + \mathsf{I}_f(Q : \mu)\\
    &= \E_{\nu}[h] - \E_{\mu}\left[h \cdot \bracket{f'^{-1}(h - \lambda_h)} \right] + \E_{\mu}\left[f\bracket{f'^{-1}(h - \lambda_h) } \right]\\
    &= \E_{\nu}[h - \lambda_h] - \E_{\mu}\left[(h - \lambda_h) \cdot \bracket{f'^{-1}(h - \lambda_h)} \right] + \E_{\mu}\left[f\bracket{f'^{-1}(h - \lambda_h) } \right]\\
    &= \E_{\nu}\left[h - \lambda_h \right] - \E_{\mu}\left[f^{\star} \circ (h - \lambda_h) \right].
\end{align}
If we denote by $\mathsf{R}(h) = \E_{\nu}[h] - \E_{\mu}[f^{\star} \circ h]$, we are able to write
\begin{align}
    \sup_{h \in \mathcal{H}} \mathsf{R}(h - \lambda_h) &= \sup_{h \in \mathcal{H}} \mathscr{L}(h,\mu_h) \\&=  \sup_{h \in \mathcal{H}} \inf_{Q \in \mathscr{P}(\Omega)} \mathscr{L}(h,Q)\\ &= \sup_{h \in \mathcal{H}} \bracket{\E_{\nu}[h] - \E_{\mu}[f^{\star} \circ h] }\\
    &= \sup_{h \in \mathcal{H}} \mathsf{R}(h).
\end{align}
Hence the optimal solution $h^{*}$ specified in the theorem statement can be decomposed into the form $h^{*} = \tilde{h} + \lambda_{\tilde{h}}$ for some  $\tilde{h} \in \mathcal{H}$. Now we will show that $\mu_{\tilde{h}}$ is in the optimal solution to
\begin{align}
    \mathcal{R}(Q) &:= d_{\mathcal{H}}(\nu, Q) + \mathsf{I}_f(Q : \mu)\\
                &= \sup_{h \in \mathcal{H}} \mathscr{L}(h,Q)
\end{align}
In order to proceed, we will define an auxillary objective:
\begin{align}
    \mathsf{J}(h) &= \E_{\nu}[h] - \E_{\mu_{\tilde{h}}}[h] + \delta_{\mathcal{H}}(h).
\end{align}
This objective corresponds to the inner IPM term inside $\mathcal{R}$. We require the following lemma to aid us in decomposing $\mathcal{R}(\mu_{\tilde{h}})$.
\begin{lemma}\label{lemma:opt-fgan}
The function $h^{*} \in \mathcal{H}$ maximizes $\mathsf{J}(h)$.
\end{lemma}
\begin{proof}
Note that by definition, $h^{*}$ maximizes the functional
\begin{align}
    \mathsf{J}_{f}(h) &= \E_{\nu}[h] - \E_{\nu}[f^{\star} \circ h] + \delta_{\mathcal{H}}(h).
\end{align}
Thus, if we take the subgradient of $-\mathsf{J}_f$ and apply \cite[~Theorem 2.97]{penot2012calculus}, the following holds via optimality of $h^{*}$
\begin{align}
    &0 \in - d\nu + (f^{\star})'(h^{*}) d\mu + \mathsf{N}_{\mathcal{H}}(h^{*}) \label{eq:optimality-cond},
\end{align}
where $\overline{h^{*}} := h^{*} + \overline{\lambda}$.
We apply the same result to $-\mathsf{J}$ to yield a condition on the optimizer $h$:
\begin{align}
    &h\text{ minimizes }-\mathsf{J}(h)\\
    &\iff 0 \in -d\nu + d\mu_{\tilde{h}} + \mathsf{N}_{\mathcal{H}}(h)\\
    &\iff 0 \in -d\nu + d\mu f'^{-1}(h^{*}  ) + \mathsf{N}_{\mathcal{H}}(h)\\
    &\stackrel{(1)}{\iff} 0 \in - d\nu + (f^{\star})'( h^{*}   ) d\mu + \mathsf{N}_{\mathcal{H}}( h  ),
\end{align}
where $(1)$ is due to the fact when $f$ is strictly convex we have $(f^{\star})' = (f')^{-1}$. Finally, note that if we set $h = h^{*}$, the condition is met by optimality condition specified in \eqref{eq:optimality-cond}.
\end{proof}

Putting all the above together, we are able to write
\begin{align}
    \inf_{Q \in \mathscr{P}(\Omega)} \mathcal{R}(Q) &\leq \mathcal{R}(\mu_{\tilde{h}})\\
    &= d_{\mathcal{H}}(\nu, \mu_{\tilde{h}}) + \mathsf{I}_f( \mu_{\tilde{h}} : \mu)\\
    &= \sup_{h \in \mathscr{F}(\Omega)} \mathsf{J}(h) + \mathsf{I}_f( \mu_{\tilde{h}} : \mu)\\
    &= \sup_{h \in \mathscr{F}(\Omega)} \mathsf{J}(h) + \E_{\mu} \left[ f\bracket{f'^{-1}\bracket{h^{*}  } } \right]\\
    &\stackrel{(1)}= \mathsf{J}(h^{*}) + \E_{\mu} \left[ f\bracket{f'^{-1}\bracket{h^{*}  } } \right]\\
    &= \E_{\nu}[h^{*}] - \E_{\mu}[f^{-1}(h^{*}  ) h^{*}] + \E_{\mu} \left[ f\bracket{f'^{-1}\bracket{h^{*}  } } \right]\\
&= \E_{\nu}[h^{*}] - \E_{\mu}\left[f^{-1}(h^{*}  ) h^{*}  - f\bracket{f'^{-1}\bracket{h^{*}} } \right]  \\
&\stackrel{(2)}= \E_{\nu}[h^{*}] - \E_{\mu}\left[ f^{\star} \circ h^{*} \right]  \\
    &\leq \sup_{h' \in \mathcal{H}} \bracket{\E_{\nu}[h'] - \E_{\mu}\left[f^{\star}\bracket{h'} \right]}\\
    &\stackrel{(3)}= \inf_{Q \in \mathscr{P}(\Omega)} \mathcal{R}(Q),
\end{align}
where $(1)$ is via the optimality of $h^{*}$ via Lemma \ref{lemma:opt-fgan}, $(2)$ is by definition of $f^{\star}$, and $(3)$ is via primal-duality as specified in \eqref{eq:gan-duality}. Finally, note that $\mu_{\tilde{h}} = \mu^{\mathcal{H}}$ by the construction of $\tilde{h}$.

\subsection{Diffusion Models}
We adapt notation from \citep{oko2023diffusion}. For some time-stamp $T \in \mathbb{N}$, we define $(B_t)_{[0,T]}$ and $\beta_t: [0,T] \to \mathbb{R}$ to denote $d$-dimensional Brownian motion and a weighting function. The forward process is defined as
\begin{align}
    d\X_t = -\beta_t \X_t dt + \sqrt{2\beta_t} dB_t, \hspace{0.3cm} \X_0 \sim \hat{P}_N,
\end{align}
where $\hat{P}_n = \sum_{j=1}^n \delta_{x_j}$ is the empirical distribution over observed $n$ samples $\braces{x_i}_{i=1}^n$. This process is referred to as the Ornstein-Uhlenbeck (OU) process whose transition density corresponds to $\X_t \mid \X_0 \sim \mathcal{N}(m_t \X_0, \sigma_t^2 I_d)$ where $m_t = \exp\bracket{-\int_{0}^t\beta_s ds}$, and $\sigma_t^2 = 1 - \exp\bracket{-2\int_{0}^t\beta_s ds}$. Denoting by $P_t = \mathcal{L}(\X_t)$, we can construct the reverse SDE process under mild conditions on $\hat{P}_N$:
\begin{align}
    d\Y_t = \beta_{T-t}\bracket{d\Y_t + 2 \nabla \log p_{T-t}(\Y_t)}dt + \sqrt{2\beta_{T-t}} dB_t,\hspace{0.3cm}d\Y_0 \sim P_{T}.
\end{align}
In practice, since we do not have access to the score function $\log p_{T-t}(Y_t)$, we approximate this with a neural network $s_{\upvartheta}(\X,t)$ for $\X \in \Omega$ and $t \in [0,T]$ and $\upvartheta$ is a learnable parameter. A Euler-Maruyama scheme is used to discretize the SDE into $K$ different steps $(\tau_k)_{k=1}^{K}$ where $\tau_0 = 0$ and $\tau_{K} = T$ and the resulting process is
\begin{align}
    d\hat{\Y}_t = \beta_{T-t}\bracket{d\hat{\Y_t} + 2 s_{\upvartheta}(\mathsf{X},T - t)}dt + \sqrt{2\beta_{T-t}} dB_t.
\end{align}
In the context of diffusion models, the distribution $\nu = \hat{P}_0 = \sum_{i=1}^n \delta_{x_i}$ where $\braces{x_i}_{i=1}^n$ is the i.i.d data $x_i \sim P$ for some population distribution $P \in \mathscr{P}(\Omega)$. As defined in Section \ref{sec:preliminaries}, we denote by $\hat{P}_t = \mathcal{L}(\X_t)$ as the law of the SDE followed by a simple Ornstein-Uhlerbeck process. For a given function $s_{\upvartheta} : \Omega \times \mathbb{R} \to \Omega$ that approximates the score function, and a discretization scheme $0 = \tau_0,\ldots,\tau_{K-1} = T$, we define the process with $k = 0,\ldots,K$ and $t \in [\tau_{k}, \tau_{k+1}]$:
\begin{align} 
    &\Y^{\mathcal{H},\upvartheta}_t = d\Y^{\mathcal{H},\upvartheta}_t + 2s_{\upvartheta}\bracket{\Y^{\mathcal{H},\upvartheta}_t, \tau_k} + 2 \nabla \log \bracket{f'^{-1}\bracket{h_{\tau_{k+1}}^{*}\bracket{\Y^{\mathcal{H},\upvartheta}_t}  } }  + \sqrt{2} dB_t,\\
    &\Y^{\mathcal{H},\upvartheta}_0 \sim \gamma_d
\end{align}
where $\gamma_d$ is a prior distribution, typically taken to be a $d$-dimensional Gaussian distribution and $h_t^{*}$ is the corresponding optimal discriminator from Theorem \ref{thm:duality-equality} when $\nu = \hat{P}_{T - t}$ and $\mu = \mathcal{L}(\Y^{\upvartheta}_t)$. The function $s_{\upvartheta}$ is a neural network approximating the score function. In the setting of $f(t) = t \log t - (t+1) \log (t+1) + 2\log 2$, this is the process used to generate samples from discriminator-guided denoised diffusion model \citep{kim2022refining}. We therefore denote the model distribution as $\mu_{T, \upvartheta}^{\mathcal{H}}= \mathcal{L}\bracket{\Y^{\mathcal{H},\upvartheta}_T}$. We are now ready to present Theorem \ref{thm:duality-equality} in the context of denoised diffusion models.
\begin{theorem}
\label{thm:equality-diffusion}
 Let $f: \mathbb{R} \to (-\infty,\infty]$ be a strictly convex lower semi-continuous and differentiable function with $f(1) = 0$.  If $f'^{-1}\bracket{t } \geq 0$ for all $t \in \operatorname{dom}(f^{\star})$ then it holds
 \begin{align}
     d_{\mathcal{H}}\bracket{\hat{P}_0,\mu_{T, \upvartheta}^{\mathcal{H}} } = \mathsf{D}_{f,\mathcal{H}}\bracket{\hat{P}_0,\mathcal{L}\bracket{\Y_T^{\upvartheta}}} - \mathsf{I}_{f}\bracket{\mu_{T, \upvartheta}^{\mathcal{H}} : \mathcal{L}\bracket{\Y_T^{\upvartheta}} }.
 \end{align}
\end{theorem}
This identity tells us that the gap (in IPM) between the refined model $\mu_{T, \upvartheta}^{\mathcal{H}}$ and data distribution $\hat{P}_0$ is precisely equal to the gap between the original SGM model without refinement  $\mathsf{D}_{f,\mathcal{H}}\bracket{\hat{P}_0,\mathcal{L}\bracket{\Y_T^{\upvartheta}}}$ minus the difference in refinement $\mathsf{I}_{f}\bracket{\mu_{T, \upvartheta}^{\mathcal{H}} : \mathcal{L}\bracket{\Y_T^{\upvartheta}} }$. We split the next part of this section into discussing these two terms.

First, note that $\mathsf{D}_{f,\mathcal{H}}\bracket{\hat{P}_0,\mathcal{L}\bracket{\Y_T^{\upvartheta}}}$ can readily be upper bounded in various ways using recent results with competing assumptions. This is attributed to the fact that $\mathsf{D}_{f,\mathcal{H}}$ lower bounds both the $f$-divergence $\mathsf{I}_f$ and the IPM $d_{\mathcal{H}}$ (\ref{boundDdI}). Therefore, we can inherit results for the reverse-KL divergence ($f(t) = -\log t$) \citep{lee2023convergence} or Total Variation ($f(t) = \card{t-1}$ \citep{chen2022sampling}, both showing fast rates of convergence. Alternatively, we can upper bound via the IPMs and take $\mathsf{D}_{f,\mathcal{H}}\bracket{\hat{P}_0,\mathcal{L}\bracket{\Y_T^{\upvartheta}}} \leq \sup_{h \in \mathcal{H}} \operatorname{Lip}(h) \cdot W_1\bracket{\hat{P}_0,\mathcal{L}\bracket{\Y_T^{\upvartheta}}}$ where $\operatorname{Lip}(h)$ is the Lipschitz constant of $h$ and $W_1$ is the $1$-Wasserstein distance. In this setting, \citep{oko2023diffusion} shows that $\mu$ approximates $\nu$ at a minimax optimal rate for the $1$-Wasserstein distance, however the term $\sup_{h \in \mathcal{H}} \operatorname{Lip}(h)$ reveals an interesting trade-off on the complexity of $\mathcal{H}$.

We now present an additional analysis by taking the assumptions of \cite{chen2022sampling}, which are considered to be minimal:
\begin{assumption}\label{ass:1}
    For all $t \geq 0$, $\X \mapsto \nabla \log \hat{P}_t(\X)$ is $L$-Lipschitz.
\end{assumption}
\begin{assumption}\label{ass:2}
    The second moment of $\hat{P}_0$ is bounded: $m_2^2 := \E_{\X \sim \hat{P}_0}\left[ \nrm{\X}^2 \right] < \infty$.
\end{assumption}
\begin{assumption}\label{ass:3}
    for all $k = 0, \ldots, K$, there exists a constant $\varepsilon_{\upvartheta} > 0$ such that
    \begin{align}
        \E_{\X \sim \hat{P}_{\tau_{k}}} \left[\nrm{ s_{\upvartheta}(\X, \tau_k) - \nabla \log \hat{P}_{\tau_k}(\X) }^2 \right] \leq \varepsilon_{\upvartheta}^2
    \end{align}
\end{assumption}
Under these assumptions, we have the following bound.
\begin{theorem}
\label{thm:tighter-convergence}
Let $f: \mathbb{R} \to (-\infty, \infty]$ be a strictly convex lower semi-continuous and differentiable function with $f(1) = 0$. Denote by $\Delta_T : \Omega \to \mathscr{P}(\Omega)$ as the forward diffusion process after $T$ steps and $\#$ the push-forward operator. Assume that $\nrm{\mathcal{H}} := \sup_{h \in \mathcal{H}} \nrm{h}_{\infty} < \infty$ and $\Omega = \mathbb{R}^d$. For $k = 1,\ldots, K$, set $\tau_k = kT/K$ (with $\tau_0 = 0$) with $s = T/K$ as the stepsize, we have under Assumptions \ref{ass:1}-\ref{ass:3}
\begin{align}
&\mathsf{D}_{f,\mathcal{H}}\bracket{\hat{P}_0,\mathcal{L}\bracket{\Y_T^{\upvartheta}}} \leq     \nrm{\mathcal{H}} \cdot \bracket{1 - \exp\bracket{-\bracket{\varepsilon_{\upvartheta}^2 + L^2 ds + L^2 m_2^2 s^2 }T}} + \mathsf{I}_f\left({\Delta_T}_{\#}  \hat{P}_0 :   \gamma_d \right).
\end{align}
\end{theorem}
This analysis follows the standard Girsanov's change of measure result utilized in \citep{chen2022sampling, oko2023diffusion} where the first term corresponds to the error incurred by approximating the score function $\varepsilon_{\upvartheta}$ and discretization error appearing with the key difference here being the term $\mathsf{I}_f\left({\Delta_T}_{\#}  \hat{P}_0 :  \gamma_d \right)$ which appears due to the general choice of $f$. In particular, when $f(t) = t \log t$ is chosen to correspond to the KL divergence, then this term decays exponentially with $T$ by the exponential convergence of Uhlerbeck-Ornstein processes studied in \citep[Theorem~5.2.1]{bakry2014analysis}. Moreover, if $f$ is chosen such that there exists some strictly increasing function $\psi_f: \mathbb{R} \to \mathbb{R}$ with $\mathsf{I}_f \leq \psi_f(\operatorname{KL})$ then we can guarantee convergence trivially by connecting the exponential convergence in KL divergence. In fact, we show in Lemma \ref{lemma:bose-einstein-KL} that when $f(t) = t \log t - (t+1) \log (t+1) + 2 \log 2$ then we directly have $\mathsf{I}_f \leq \operatorname{KL}$ giving us the required convergence. We show that we can relate $f$-divergences to the KL divergence.
\begin{lemma}\label{lemma:fdiv-kl}
    Suppose $f: \mathbb{R} \to (-\infty, \infty]$ is a strictly convex differentiable lower semi-continuous function with $f(1) = 0$ and denote by $\Delta_T: \Omega \to \mathscr{P}(\Omega)$ as the forward diffusion process after $T$ steps. Let $r_T = d({\Delta_T}_{\#} \hat{P}_0) / d\gamma_d$ be the Radon-Nikodym derivative then we have $\mathsf{I}_f\left({\Delta_T}_{\#}  \hat{P}_0 :   \gamma_d \right) \leq \sup_{\X \in \Omega} \card{f'\bracket{r_T(\X)}} \cdot \sqrt{\operatorname{KL\left({\Delta_T}_{\#}  \hat{P}_0 :   \gamma_d \right)}}$.
\end{lemma}
In this result, we can thus see that convergence of the forward diffusion process in $\mathsf{I}_f$ is as fast as the KL divergence except for an extra term that depends on the density ratio $r_T$. Indeed, we assume that as $T$ gets large, then $r_T \to 1$. We now move onto the second refinement term from Theorem \ref{thm:equality-diffusion}: $\mathsf{I}_{f}\bracket{\mu_{T, \upvartheta}^{\mathcal{H}} : \mathcal{L}\bracket{\Y_T^{\upvartheta}} }$. Note that this quantity is subtracted from the total gap and is positive. In order to better understand this gain, we reparametrize it into a $1$-dimensional integral: suppose the classifier decomposes with the natural composite link: $h^{*} = f'(\eta_{h^{*}}/ (1 - \eta_{h^{*}}))$ where $\eta_{h^{*}} : \Omega \to [0,1]$ is the class probability estimate.
\begin{theorem}
    \label{thm:gain-2}
    Let $f: \mathbb{R} \to (-\infty,\infty]$ be a strictly convex lower semi-continuous and differentiable function with $f(1) = 0$. Denote by $h^{*} \in \mathcal{H}$ as the optimal discriminator and $\mu_{T, \upvartheta}^{\mathcal{H}}$ such that $d\mu_{T, \upvartheta}^{\mathcal{H}}/d\mu = f'^{-1}(h^{*} )$ be the optimal solutions to \eqref{eq:gan-duality} . If $f'^{-1}\bracket{t } \geq 0$ for all $t \in \operatorname{dom}(f^{\star})$ then we have
    \begin{align}
        \mathsf{I}_{f}\bracket{\mu_{T, \upvartheta}^{\mathcal{H}} : \mathcal{L}\bracket{\Y_T^{\upvartheta}}} = \int_0^1 f\bracket{\frac{t}{1 - t}}d\rho_{h^{*}}(t)
    \end{align}
    where $\mu_T^{\upvartheta} = \mathcal{L}\bracket{\Y_T^{\upvartheta}}$ and $\eta_{h^{*}}$ is such that $h^{*} = f'(\eta_{h^{*}} / (1 - \eta_{h^{*}}))$ and $\rho_{h^{*}} := {\eta_{h^{*}}}_{\#} \mu_{T}^{\upvartheta}$.
\end{theorem}
In order to understand this quantity, note that $\rho_{h^{*}}$ corresponds to the regions in $[0,1]$ where the classifier predicts $\mu_T^{\upvartheta}$ to be. Let us consider the case of binary-cross entropy corresponding to \citep{kim2022refining} where $f(t) =  t \log t - (t+1) \log(t+1) + 2 \log 2$. The expression then simplifies to  
\begin{align}
    \int_0^1  \bracket{\frac{t}{1-t} \log t + \log(1 - t) + 2\log 2}d\rho_{h^{*}}(t).
\end{align}
Recalling that $\eta_{h^{*}}(\X)$ corresponds to the class probability estimation of a point $\X \in \Omega$ belonging to the class $\hat{P}_0$ (as opposed to $\mu_T^{\upvartheta}$), we can say that if $\mathcal{H}$ is rich enough to classify between $\hat{P}_0$ and $\mu_T^{\upvartheta}$ then $\eta_{h^{*}}$ is close to $0$ on the support of $\mu_{T}^{\upvartheta}$ and thus $\rho_{h^{*}}(t)$ is concentrated around $0$, thereby increasing the above expression to $f(0) = 2\log 2$, the maximal possible value for $\mathsf{I}_{f}\bracket{\mu_{T, \upvartheta}^{\mathcal{H}} : \mathcal{L}\bracket{\Y_T^{\upvartheta}}}$. On the other hand, if the best classifier $\eta_{h^{*}}$ is unable to classify the two classes, then $\eta_{h^{*}} \to 1/2$ (a random classifier) and thus the above expression tends to $f(1) = 0$. Therefore, the gain in refinement can be understood by the discriminative abilities of $\mathcal{H}$, and a richer choice for $\mathcal{H}$ will lead to more improvement. We will see in what follows how this forms a trade-off with generalization.

\subsection{Proof of Theorem \ref{thm:tighter-convergence}}
For any timestep $T > 0$, let $\Delta_T : \Omega \to \mathscr{P}(\Omega)$ denote the forward diffusion process dictated by the Ornstein-Uhlerbeck process and let $\overleftarrow{\Delta_T}$ be the reverse process. Furthermore, let $\overleftarrow{\mathsf{S}_{\upvartheta}}$ be the process after following the DDPM algorithm at $T$ steps. We then have
\begin{align}
    \hat{P}_0 = {\overleftarrow{\Delta_T}}_{\#} {\Delta_T}_{\#}  \hat{P}_0\\
    \mathcal{L}(\mathsf{Y}_T^{\upvartheta}) = \overleftarrow{\mathsf{S}_{\upvartheta}} \# \gamma_d.
    \end{align}
We then have
\begin{align}
    &\mathsf{D}_{f,\mathcal{H}}\left( \hat{P}_0, \mathcal{L}\left(\Y_T^{\upvartheta}\right)\right)\\&= \mathsf{D}_{f,\mathcal{H}}\left({\overleftarrow{\Delta_T}}_{\#} {\Delta_T}_{\#}  \hat{P}_0, \overleftarrow{\mathsf{S}_{\upvartheta}} \# \gamma_d\right)\\
    &= \inf_{\overline{\mu} \in \mathscr{P}(\Omega)}\bracket{d_{\mathcal{H}}\bracket{{\overleftarrow{\Delta_T}}_{\#} {\Delta_T}_{\#}  \hat{P}_0,\overline{\mu}} + \mathsf{I}_f\left(\overline{\mu} : \overleftarrow{\mathsf{S}_{\upvartheta}} \# \gamma_d \right)   }\\
    &\leq d_{\mathcal{H}}\bracket{{\overleftarrow{\Delta_T}}_{\#} {\Delta_T}_{\#}  \hat{P}_0,\overleftarrow{\mathsf{S}_{\upvartheta}}_{\#} {\Delta_T}_{\#}  \hat{P}_0 } + \mathsf{I}_f\left(\overleftarrow{\mathsf{S}_{\upvartheta}}_{\#} {\Delta_T}_{\#}  \hat{P}_0 : \overleftarrow{\mathsf{S}_{\upvartheta}} \# \gamma_d \right)\\
    &\stackrel{(1)}\leq \nrm{\mathcal{H}} \cdot \operatorname{TV}\bracket{{\overleftarrow{\Delta_T}}_{\#} {\Delta_T}_{\#}  \hat{P}_0,\overleftarrow{\mathsf{S}_{\upvartheta}}_{\#} {\Delta_T}_{\#}  \hat{P}_0} + \mathsf{I}_f\left(\overleftarrow{\mathsf{S}_{\upvartheta}}_{\#} {\Delta_T}_{\#}  \hat{P}_0 : \overleftarrow{\mathsf{S}_{\upvartheta}} \# \gamma_d \right)\\
    &\stackrel{(2)}\leq \nrm{\mathcal{H}} \cdot \operatorname{TV}\bracket{{\overleftarrow{\Delta_T}}_{\#} {\Delta_T}_{\#}  \hat{P}_0,\overleftarrow{\mathsf{S}_{\upvartheta}}_{\#} {\Delta_T}_{\#}  \hat{P}_0} + \mathsf{I}_f\left({\Delta_T}_{\#}  \hat{P}_0 : \gamma_d \right)\\
    &\stackrel{(3)}\leq \nrm{\mathcal{H}} \cdot \bracket{1 - \exp\bracket{-\operatorname{KL}\bracket{{{\overleftarrow{\Delta_T}}_{\#} {\Delta_T}_{\#}  \hat{P}_0 :\overleftarrow{\mathsf{S}_{\upvartheta}}_{\#} {\Delta_T}_{\#}  \hat{P}_0}}}} + \mathsf{I}_f\left({\Delta_T}_{\#}  \hat{P}_0 : \gamma_d \right)\\
    &\stackrel{(4)}\leq \nrm{\mathcal{H}} \cdot \bracket{1 - \exp\bracket{-\bracket{\varepsilon^2 + L^2 ds + L^2 m_2^2 s^2 }T}} + \mathsf{I}_f\left({\Delta_T}_{\#}  \hat{P}_0 : \gamma_d \right),
\end{align}
where $(1)$ is due to the fact that $d_{\mathcal{H}} \leq \nrm{\mathcal{H}} \cdot \operatorname{TV}$, $(2)$ is via the data processing inequality of $f$-divergences, $(3)$ is by \citep[Lemma~2.1]{bretagnolle1979estimation}, and $(4)$ is via \cite[Theorem~9]{chen2022sampling} under the assumptions.
\subsection{Proof of Lemma \ref{lemma:fdiv-kl}}
Using the variational form of $f$-divergence, we know the witness is attained at $w^{*} = f'(r_T)$ using \citep{nguyen2010estimating}, let $\nrm{w^{*}} = \sup_{\X \in \Omega} \card{f'(r_T(\X))}$ :
\begin{align}
    \mathsf{I}_f({\Delta_T}_{\#} \hat{P}_0 : \gamma_d) &= \sup_{w: \Omega \to \operatorname{dom}(f^{\star}) }\bracket{\E_{{\Delta_T}_{\#} \hat{P}_0}[w] - \E_{\gamma_d}[f^{\star} \circ w] }\\
    &= \E_{{\Delta_T}_{\#} \hat{P}_0}[w^{*}] - \E_{\gamma_d}[f^{\star} \circ w^{*}] \\
    &\stackrel{(1)}{\leq} \E_{{\Delta_T}_{\#} \hat{P}_0}[w^{*}] - \E_{\gamma_d}[w^{*}] \\
    &=\nrm{w^{*}} \cdot \bracket{ \E_{{\Delta_T}_{\#} \hat{P}_0}\left[\frac{w^{*}}{\nrm{w^{*}}}\right] - \E_{\gamma_d}\left[\frac{w^{*}}{\nrm{w^{*}}}\right] }\\
    &\leq \nrm{w^{*}} \cdot\sup_{h : \nrm{h} \leq 1} \bracket{ \E_{{\Delta_T}_{\#} \hat{P}_0}\left[h\right] - \E_{\gamma_d}\left[h\right] }\\
    &\stackrel{(2)}{\leq} \nrm{w^{*}} \cdot \sqrt{\operatorname{KL}\bracket{{\Delta_T}_{\#} \hat{P}_0 : \gamma_d} },
\end{align}
where $(1)$ is due to the fact that $f^{\star}(t) = \sup_{t'}\bracket{t \cdot t' - f(t')} \geq t - f(1) = t$ and $(2)$ is via Pinsker's inequality.

\subsection{Proof of Theorem \ref{thm:gain}}
We can then write
\begin{align}
    \mathsf{I}_f  &= \E_{\X \sim \mu} \left[ f\bracket{f'^{-1}\bracket{h^{*}(\X)  } } \right]\\
    &= \E_{\X \sim \mu} \left[ f\bracket{f'^{-1}\bracket{ f'\bracket{\frac{\eta_{h^{*}}(\X)}{1 - \eta_{h^{*}}(\X)}} } } \right]\\
    &= \E_{\X \sim \mu} \left[ f\bracket{\frac{\eta_{h^{*}}(\X)}{1 - \eta_{h^{*}}(\X)}} \right]\\
    &= \E_{t \sim \rho_{h^{*}}} \left[ f\bracket{\frac{t}{1 - t}} \right]\\
    &= \int_{0}^1 f\bracket{\frac{t}{1 - t}} d\rho_{h^{*}}(t).
\end{align}

\subsection{Proof of Theorem \ref{thm:generalization}}
Using a standard application of McDiarmind's inequality such as in \citep[Lemma~5]{Bousquet2004} and \citep[Theorem~3.1]{zhang2017discrimination}, we have
\begin{align}
    d_{\mathcal{H}}(P, \hat{P}) \leq \mathscr{R}_n(\mathcal{H}) + 2 \nrm{\mathcal{H}} \cdot \sqrt{\frac{1}{2n} \ln\bracket{\frac{1}{\delta}}},\label{eq:ipm-to-rademacher}
\end{align}
 with probability $1 - \delta$. Thus, we have
\begin{align}
    d_{\mathcal{H}}\bracket{P,\mu^{\mathcal{H}} } &= \sup_{h \in \mathcal{H}} \bracket{ \E_{P}[h] - \E_{\mu^{\mathcal{H}}}[h] }\\
    &\leq \sup_{h \in \mathcal{H}} \bracket{ \E_{P}[h] - \E_{\hat{P}}[h] + \E_{\hat{P}}[h] - \E_{\mu^{\mathcal{H}}}[h] }\\
    &\leq \sup_{h \in \mathcal{H}} \bracket{ \E_{P}[h] - \E_{\hat{P}}[h]} + \sup_{h \in \mathcal{H}}\bracket{\E_{\hat{P}}[h] - \E_{\mu^{\mathcal{H}}}[h] }.
\end{align}
The first term can be bounded by \eqref{eq:ipm-to-rademacher} and the second term can be decomposed via the main identity in Equation \eqref{thm:main-identity}
\subsection{Derivation for Example 1}
We focus on the setting described for Variational Inference where we will show that
\begin{align}
    \sup_{b \in \mathbb{R}} \bracket{b - \E_{\mu}[\exp\bracket{-L + b - 1}] } = \log \E_{\mu} \left[\exp\bracket{-L} \right].
\end{align}
First we set $M = \E_{\mu}\left[\exp(-L - 1) \right]$ then note that 
\begin{align}
    b - \E_{\mu}[\exp\bracket{-L + b - 1}] &=  b - e^b \E_{\mu}[\exp\bracket{-L - 1}]\\
    &= b - e^b M.
\end{align}
Differentiating this objective with respect to $b$ yields the optimal $b$ is $b^{*} = -\log M$.

\subsection{Proof of Example 2}
We first derive the conjugate of $f(t)$:
\begin{align}
    f^{\star}(t) &= \sup_{t'} \bracket{t \cdot t' - f(t')}\\
    &= \sup_{t'}\bracket{t \cdot t' - t' \log(t') + (t'+1) \log (t'+1) - 2 \log 2 }\\
    &= \begin{cases}-2 \log 2 -\log(1 - \exp(t)) &\text{ 
 if }t < 0 \\ \infty &\text{ if }t \geq 0. \end{cases}
\end{align}
Therefore, we have
\begin{align}
    \mathsf{D}_{f,\mathcal{H}}(\nu,\mu) &= \sup_{h \in \mathcal{H}} \bracket{\E_{\nu}[h] - \E_{\mu}[f^{\star} \circ h] }\\
    &= \sup_{\theta \in \Theta}  \bracket{\E_{\nu}[\log(\eta_{\theta})]  - \E_{\mu}[f^{\star} \circ \bracket{\log(\eta_{\theta}) }] }\\
    &= \sup_{\theta \in \Theta}  \bracket{\E_{\nu}[\log(\eta_{\theta})]  + \E_{\mu}[ \log(1 - \eta_{\theta}) ]} + 2\log 2\\
    &= -\inf_{\theta \in \Theta}  \bracket{\E_{\nu}[-\log(\eta_{\theta})]  + \E_{\mu}[ -\log(1 - \eta_{\theta}) ] } + 2\log 2.
\end{align}

\subsection{Additional Results}
\begin{lemma}\label{lemma:bose-einstein-KL}
    For any $\mu, \nu \in \mathscr{P}(\Omega)$, if $f(t) = t \log t - (t+1) \log(t+1) + 2 \log 2$, we have that $\mathsf{I}_f(\mu : \nu) \leq \operatorname{KL}(\mu : \nu)$.
\end{lemma}
\begin{proof}
Let $f_{\operatorname{KL}}(t) = t \log t$ and $f_{\mathrm{excess}}(t) = 2 \log 2 - (t+1) \log (t+1)$ then we have
\begin{align}
    \mathsf{I}_f(\mu : \nu) &= \E_{\nu} \left[ f(d\mu/d\nu)\right]\\
    &=\E_{\nu} \left[ f_{\operatorname{KL}}(d\mu/d\nu) + f_{\mathrm{excess}}(d\mu/d\nu)\right]\\
    &= \E_{\nu} \left[ f_{\operatorname{KL}}(d\mu/d\nu)\right] + \E_{\nu} \left[f_{\mathrm{excess}}(d\mu/d\nu)\right]\\
    &\stackrel{(1)}{\leq} \E_{\nu} \left[ f_{\operatorname{KL}}(d\mu/d\nu)\right] + f_{\mathrm{excess}}\bracket{\E_{\nu} \left[(d\mu/d\nu)\right]}\\
    &= \E_{\nu} \left[ f_{\operatorname{KL}}(d\mu/d\nu)\right] + f_{\mathrm{excess}}\bracket{1}\\
    &= \E_{\nu} \left[ f_{\operatorname{KL}}(d\mu/d\nu)\right],
\end{align}
where $(1)$ is via Jensen's inequality, noting that $f_{\operatorname{excess}}$ is concave.
\end{proof}

The main duality result requires $\mathcal{H}$ to be closed under additive constants however we look to generalizing this, to get a better understanding of how this plays a role in discriminator guidance. We first consider a variant of Lemma \ref{lem:opt-mu} that relaxes the constraint of $Q$ being a probability measure and even $f'^{-1}(t) \geq 0$.
\begin{lemma}\label{lem:opt-fipm-general}
    Let $f: \mathbb{R} \to (-\infty, \infty]$ be a lower semi-continuous convex function with $f(1) = 0$. For a fixed $h \in \mathscr{F}(\Omega,\mathbb{R})$, we have
    \begin{align}
        \mu_h \in \arginf_{\mu \in \mathscr{B}(\Omega)} \mathscr{L}(h, Q) \implies \frac{d\mu_h}{d\mu} = f'^{-1}(h).
    \end{align}
\end{lemma}
\begin{proof}
Similar to the proof of Lemma \ref{lem:opt-mu}, we can reparametrize $Q$ with $r$: $Q = r \cdot \mu$ however we have no other restriction on $r$. Thus, we have:
\begin{align}
    \mathscr{L}(Q,\mu) = \E_{\nu}[h] - \E_{\mu}[r \cdot h] + \E_{\mu}[f(r)],
\end{align}
and differentiating with respect to $r$ and setting the derivative to zero yields:
\begin{align}
    &0 = -h + f'(r)\\
    &\implies r = f'^{-1}(h).
\end{align}
\end{proof}
Furthermore, we have that
\begin{align}
    \sup_{h \in \mathcal{H}} \inf_{Q \in \mathscr{B}(\Omega)} \mathscr{L}(h,\mu) &= \sup_{h \in \mathcal{H}} \inf_{\mu \in \mathscr{B}(\Omega)} \braces{ \E_{\nu}[h] - \E_{Q}[h] + \mathsf{I}_f(Q : \mu) }\\
    &= \sup_{h \in \mathcal{H}} \braces{ \E_{\nu}[h] - \sup_{\mu \in \mathscr{B}(\Omega)}\braces{\E_{Q}[h] - \mathsf{I}_f(Q : \mu)} }\\
    &\stackrel{(1)}{=}\sup_{h \in \mathcal{H}} \braces{ \E_{\nu}[h] - \E_{\mu}[f^{\star} \circ h] },
\end{align}
where $(1)$ is due to the fact that $E_{\mu}[f^{\star} \circ h]$ is the Legendre-Fenchel dual of $Q \mapsto \mathsf{I}_f(Q : \mu)$. In order to see this, note that $Q \mapsto \mathsf{I}_f(Q : \mu)$ is proper, convex and lower semicontinuous and thus by Theorem \ref{self-conjugacy}, it suffices to show $Q \mapsto \mathsf{I}_f(Q : \mu)$ is the Legendre-Fenchel dual of $h \mapsto \mathscr{K}(h) := \E_{\mu}[f^{\star} \circ h]$:
\begin{align}
    \mathscr{K}^{\star}(Q) &= \sup_{h \in \mathscr{F}(\Omega,\mathbb{R})} \braces{\E_{Q}[h] - \mathscr{K}(h) }\\
    &= \sup_{h \in \mathscr{F}(\Omega,\mathbb{R})} \braces{\E_{Q}[h] - \E_{\mu}[f^{\star} \circ h] }\\
    &\stackrel{(1)}= \mathsf{I}_{f}(Q : \mu), \label{eq:dual-of-fdiv}
\end{align}
where $(1)$ is due to the variational formulation of $\mathsf{I}_f$ \citep{nguyen2010estimating}. Note that we do not require $\mathcal{H}$ to be closed under additive constants, the only caveat is that the set $\mathscr{B}(\Omega)$ is not compact, meaning that we cannot apply a minimax theorem to get strong duality. However under the assumption of the existence of a compact set, we can get the result for $\mathcal{H}$ that are not additive.
\begin{assumption}
Let $f: \mathbb{R} \to (-\infty, \infty]$ be a lower semi-continuous convex function with $f(1) = 0$. For any set of functions $\mathcal{H} \subseteq \mathscr{F}(\Omega, \mathbb{R})$, suppose there exists a convex compact set $\mathsf{B} \subseteq \mathscr{B}(\Omega)$ such that $\mathscr{P}(\Omega) \subset \mathsf{B}$ and
\begin{align}
    \braces{\mu_h : \frac{d\mu_h}{d\mu} = f'^{-1}(h)} \subset \mathsf{B}.
\end{align}
\end{assumption}
One way of having such an assumption satisfied is if each $\mu_h$ satisfies $\mu_h(\Omega) \leq 1$ then $\mathsf{B}$ can be the unit ball in $\mathscr{B}(\Omega)$ which is compact under the vague topology by the Banach-Alaoglu theorem. Note that if we pick $\mathcal{H}$ to be parametrized in the propoer composite loss framework \citep{sEO}:
\begin{align}
    \mathcal{H} = \braces{f'\bracket{\frac{\eta_{\theta}}{1 - \eta_{\theta}}} : \theta \in \Theta },
\end{align}
where $\theta \mapsto \eta_{\theta} \in [0,1]$ is an arbitrarily parametrized function such as a deep neural network then we have
\begin{align}
    \mu_h(\Omega) \leq 1 \iff \E_{\mu}\left[ \frac{\eta_{\theta}}{1 - \eta_{\theta}}\right] \leq 1. \label{eq:cpe-compactness}
\end{align}
Since $\eta_{\theta}$ is the class probability estimate of a point to be \textit{not} in the support of $\mu$ (since we discriminate between $\mu$ and $\nu$), we can expect the parametrized models to mostly satsify this property. Under this assumption, we have a generalized duality.
\begin{theorem}
Let $f: \mathbb{R} \to (-\infty, \infty]$ be a lower semi-continuous convex function with $f(1) = 0$. For any set of convex functions $\mathcal{H} \subseteq \mathscr{F}(\Omega, \mathbb{R})$, suppose there exists $\mathsf{B}$ from the above Assumption, then we have
\begin{align}
    \sup_{h \in \mathcal{H}}\braces{\E_{\nu}[h] - \E_{\mu}[f^{\star} \circ h] } = \inf_{Q \in \mathsf{B}}\braces{ d_{\mathcal{H}}(\nu,Q) + \mathsf{I}(Q : \mu) }.
\end{align}
\end{theorem}
\begin{proof}
Note that from \eqref{eq:dual-of-fdiv} we have
\begin{align}
  \sup_{h \in \mathcal{H}}\braces{\E_{\nu}[h] - \E_{\mu}[f^{\star} \circ h] } &= \sup_{h \in \mathcal{H}} \braces{ \E_{\nu}[h] - \sup_{Q \in \mathscr{B}(\Omega)}\braces{\E_{Q}[h] - \mathsf{I}_f(Q : \mu)} }\\   
  &= \sup_{h \in \mathcal{H}} \inf_{Q \in \mathscr{B}(\Omega)}\braces{ \E_{\nu}[h] - \E_{Q}[h] + \mathsf{I}_f(Q : \mu)} \\   
  &\stackrel{(1)}{=} \sup_{h \in \mathcal{H}} \inf_{Q \in \mathsf{B}}\braces{ \E_{\nu}[h] - \E_{Q}[h] + \mathsf{I}_f(Q : \mu)} \\ 
  &\stackrel{(2)}{=}  \inf_{Q \in \mathsf{B}} \sup_{h \in \mathcal{H}}\braces{ \E_{\nu}[h] - \E_{Q}[h] + \mathsf{I}_f(Q : \mu)} \\
  &= \inf_{Q \in \mathsf{B}}\braces{ d_{\mathcal{H}}(\nu,Q) + \mathsf{I}(Q : \mu) },
\end{align}
where $(1)$ is due to the fact that $\mathsf{B}$ contains the optimal measure by Lemma \ref{lem:opt-fipm-general} and $(2)$ is due to the fact that since $\mathcal{H}$ and $\mathsf{B}$ are convex and $\mathsf{B}$ is compact, we are able to apply Ky Fan's minimax Theorem \citep{fan1953minimax} in the same way as \citep[Lemma~27]{liu2018inductive}.
\end{proof}
Putting this Theorem to use, we have that
\begin{align}
    \sup_{h \in \mathcal{H}}\braces{\E_{\nu}[h] - \E_{\mu}[f^{\star} \circ h] } = \inf_{Q \in \mathsf{B}}\braces{ d_{\mathcal{H}}(\nu,Q) + \mathsf{I}(Q : \mu) } \leq \inf_{Q \in \mathscr{P}(\Omega)}\braces{ d_{\mathcal{H}}(\nu,Q) + \mathsf{I}(Q : \mu) }.
\end{align}

Note that this inequality becomes tight when $\mathcal{H}$ is large enough to be closed under addition however the optimal \textit{refined} measure appearing in the optimization problem still satisfies the form
\begin{align}
    \frac{d\mu^{\mathcal{H}}}{d\mu} = f'^{-1}(h^{*}),
\end{align}
where $h^{*}$ is the optimal function from the dual problem. The only difference here is that the measure $\mu^{\mathcal{H}}$ may not necessarily be a probability measure. Thus in practice, if $\mathcal{H}$ is not closed under addition, it is intuitive to compute $f'^{-1}(h^{*})$ and normalize it as a heuristic.

We now consider deriving the full framework for $f(t) = t \log t$, the KL-divergence. In this case, note that $f$ is still strictly convex, with $f'^{-1}(t) \geq 0$. Next, note that $f^{\star}(t) = \exp(t-1)$, thus the discriminator task is
\begin{align}
    \mathsf{D}_{f,\mathcal{H}} = \sup_{h \in \mathcal{H}} \bracket{\E_{\nu}\left[h \right] - \E_{\mu}[\exp(h-1)] },
\end{align}
and the final refined distribution will be:
\begin{align}
    \mu^{\mathcal{H}} = \mu \cdot \exp(h^{*}) / \E_{\mu}[h^{*}]. \label{bde}
\end{align}
We note that boosted density estimation algorithms have been developed using discriminators in combination with the expression in \citep{cranko2019boosted,husain2020local,soen2020data}.

%% file: example_paper.bib
@STRING{ARXIV = "ArXiv"}

@STRING{IS = "Information Systems"}

@STRING{JASA = "J. of the Am. Stat. Assoc."}

@STRING{JMLR = "JMLR"}

@STRING{JOTA = "J. of Optimization Theory and Applications"}

@ARTICLE{sEO,
      TITLE = "Elicitation of personal probabilities and expectations",
      AUTHOR = "L.-J. Savage",
      JOURNAL = JASA,
      PAGES = "783--801",
      YEAR = "1971"}

@article{bartlett1998sample,
  title={The sample complexity of pattern classification with neural networks: the size of the weights is more important than the size of the network},
  author={Bartlett, Peter L},
  journal={IEEE transactions on Information Theory},
  volume={44},
  number={2},
  pages={525--536},
  year={1998},
  publisher={IEEE}
}

@article{rockafellar1968integrals,
  title={Integrals which are convex functionals},
  author={Rockafellar, Ralph},
  journal={Pacific journal of mathematics},
  volume={24},
  number={3},
  pages={525--539},
  year={1968},
  publisher={Mathematical Sciences Publishers}
}

@book{zalinescu2002convex,
  title={Convex analysis in general vector spaces},
  author={Zalinescu, Constantin},
  year={2002},
  publisher={World scientific}
}

@article{liu2018inductive,
  title={The inductive bias of restricted f-gans},
  author={Liu, Shuang and Chaudhuri, Kamalika},
  journal={arXiv preprint arXiv:1809.04542},
  year={2018}
}

@inproceedings{husain2019primal,
  title={A Primal-Dual link between GANs and Autoencoders},
  author={Husain, Hisham and Nock, Richard and Williamson, Robert C},
  booktitle={Advances in Neural Information Processing Systems},
  pages={413--422},
  year={2019}
}

@inproceedings{nmoAS,
  author    = {R. Nock and
               A.-K. Menon and
               C.-S. Ong},
  title     = {A scaled {B}regman theorem with applications},
  booktitle = {Advances in Neural Information Processing Systems},
  pages     = {19--27},
  year      = {2016}
}

@Article{mOAI,
  author   = "P. Mar{\'e}chal",
  title    = "On a functional operation generating convex functions, part 1: duality",
  journal  = JOTA,
  year     = {2005},
  volume   = {126},
  pages    = {175-189},
  }

@Article{mOAII,
  author   = "P. Mar{\'e}chal",
  title    = "On a functional operation generating convex functions, part 2: algebraic properties",
  journal  = JOTA,
  year     = {2005},
  volume   = {126},
  pages    = {375-366},
  }

@article{rwCB,
  author    = {M.-D. Reid and
               R.-C. Williamson},
  title     = {Composite Binary Losses},
  journal   = JMLR,
  pages = "2387--2422",
  volume    = {11},
  year      = {2010}
}

@article{rwID,
  author    = {M.-D. Reid and
               R.-C. Williamson},
  title     = "Information, Divergence and Risk for Binary Experiments",
  journal   = JMLR,
  volume    = {12},
  pages     = {731--817},
  year      = {2011}
}

@article{bartlett2002rademacher,
  title={Rademacher and Gaussian complexities: Risk bounds and structural results},
  author={Bartlett, Peter L and Mendelson, Shahar},
  journal={Journal of Machine Learning Research},
  volume={3},
  number={Nov},
  pages={463--482},
  year={2002}
}

@inproceedings{nock2017f,
  title={f-GANs in an information geometric nutshell},
  author={Nock, Richard and Cranko, Zac and Menon, Aditya K and Qu, Lizhen and Williamson, Robert C},
  booktitle={Advances in Neural Information Processing Systems},
  pages={456--464},
  year={2017}
}

@inproceedings{goodfellow2014generative,
  title={Generative adversarial nets},
  author={Goodfellow, Ian and Pouget-Abadie, Jean and Mirza, Mehdi and Xu, Bing and Warde-Farley, David and Ozair, Sherjil and Courville, Aaron and Bengio, Yoshua},
  booktitle={Advances in neural information processing systems},
  pages={2672--2680},
  year={2014}
}

@inproceedings{liu2017approximation,
  title={Approximation and convergence properties of generative adversarial learning},
  author={Liu, Shuang and Bousquet, Olivier and Chaudhuri, Kamalika},
  booktitle={Advances in Neural Information Processing Systems},
  pages={5545--5553},
  year={2017}
}

@article{zhang2017discrimination,
  title={On the discrimination-generalization tradeoff in GANs},
  author={Zhang, Pengchuan and Liu, Qiang and Zhou, Dengyong and Xu, Tao and He, Xiaodong},
  journal={arXiv preprint arXiv:1711.02771},
  year={2017}
}

@article{nguyen2010estimating,
  title={Estimating divergence functionals and the likelihood ratio by convex risk minimization},
  author={Nguyen, XuanLong and Wainwright, Martin J and Jordan, Michael I},
  journal={IEEE Transactions on Information Theory},
  volume={56},
  number={11},
  pages={5847--5861},
  year={2010},
  publisher={IEEE}
}

@inproceedings{nowozin2016f,
  title={f-gan: Training generative neural samplers using variational divergence minimization},
  author={Nowozin, Sebastian and Cseke, Botond and Tomioka, Ryota},
  booktitle={Advances in Neural Information Processing Systems},
  pages={271--279},
  year={2016}
}

@book{penot2012calculus,
  title={Calculus without derivatives},
  author={Penot, Jean-Paul},
  volume={266},
  year={2012},
  publisher={Springer Science \& Business Media}
}

@inproceedings{farnia2018convex,
  title={A convex duality framework for GANs},
  author={Farnia, Farzan and Tse, David},
  booktitle={Advances in Neural Information Processing Systems},
  pages={5248--5258},
  year={2018}
}

@article{fan1953minimax,
  title={Minimax theorems},
  author={Fan, Ky},
  journal={Proceedings of the National Academy of Sciences of the United States of America},
  volume={39},
  number={1},
  pages={42},
  year={1953},
  publisher={National Academy of Sciences}
}

@inproceedings{goodfellow2014explaining,
  author    = {Ian J. Goodfellow and
               Jonathon Shlens and
               Christian Szegedy},
  editor    = {Yoshua Bengio and
               Yann LeCun},
  title     = {Explaining and Harnessing Adversarial Examples},
  booktitle = {3rd International Conference on Learning Representations, {ICLR} 2015,
               San Diego, CA, USA, May 7-9, 2015, Conference Track Proceedings},
  year      = {2015},
}

@book{rockafellar1970convex,
  title={Convex analysis},
  author={Rockafellar, R Tyrrell},
  number={28},
  year={1970},
  publisher={Princeton university press}
}

@article{husain2020distributional,
  title={Distributional {R}obustness with {IPM}s and links to {R}egularization and {GAN}s},
  author={Husain, Hisham},
  journal={Advances in Neural Information Processing Systems},
  volume={33},
  year={2020}
}

@book{Bousquet2004,
	author={Bousquet, Olivier and Boucheron, St{\'e}phane and Lugosi, G{\'a}bor},
	title={Introduction to Statistical Learning Theory},
	year={2004},
	publisher={Springer Berlin Heidelberg},
	address={Berlin, Heidelberg},
	pages={169--207},
	isbn={978-3-540-28650-9},
	doi={10.1007/978-3-540-28650-9_8},
	url={https://doi.org/10.1007/978-3-540-28650-9_8}
}

@misc{liang2016cs229t,
  title={CS229T/STAT231: Statistical Learning Theory (Winter 2016)},
  author={Liang, Percy},
  year={2016}
}

@article{oko2023diffusion,
  title={Diffusion models are minimax optimal distribution estimators},
  author={Oko, Kazusato and Akiyama, Shunta and Suzuki, Taiji},
  journal={arXiv preprint arXiv:2303.01861},
  year={2023}
}

@article{chen2022sampling,
  title={Sampling is as easy as learning the score: theory for diffusion models with minimal data assumptions},
  author={Chen, Sitan and Chewi, Sinho and Li, Jerry and Li, Yuanzhi and Salim, Adil and Zhang, Anru R},
  journal={arXiv preprint arXiv:2209.11215},
  year={2022}
}

@article{lee2022convergence,
  title={Convergence for score-based generative modeling with polynomial complexity},
  author={Lee, Holden and Lu, Jianfeng and Tan, Yixin},
  journal={Advances in Neural Information Processing Systems},
  volume={35},
  pages={22870--22882},
  year={2022}
}

@inproceedings{lee2023convergence,
  title={Convergence of score-based generative modeling for general data distributions},
  author={Lee, Holden and Lu, Jianfeng and Tan, Yixin},
  booktitle={International Conference on Algorithmic Learning Theory},
  pages={946--985},
  year={2023},
  organization={PMLR}
}

@article{li2023towards,
  title={Towards Faster Non-Asymptotic Convergence for Diffusion-Based Generative Models},
  author={Li, Gen and Wei, Yuting and Chen, Yuxin and Chi, Yuejie},
  journal={arXiv preprint arXiv:2306.09251},
  year={2023}
}

@article{de2022convergence,
  title={Convergence of denoising diffusion models under the manifold hypothesis},
  author={De Bortoli, Valentin},
  journal={arXiv preprint arXiv:2208.05314},
  year={2022}
}

@article{kim2022refining,
  title={Refining generative process with discriminator guidance in score-based diffusion models},
  author={Kim, Dongjun and Kim, Yeongmin and Kang, Wanmo and Moon, Il-Chul},
  journal={arXiv preprint arXiv:2211.17091},
  year={2022}
}

@article{chung2022score,
  title={Score-based diffusion models for accelerated MRI},
  author={Chung, Hyungjin and Ye, Jong Chul},
  journal={Medical image analysis},
  volume={80},
  pages={102479},
  year={2022},
  publisher={Elsevier}
}

@article{batzolis2021conditional,
  title={Conditional image generation with score-based diffusion models},
  author={Batzolis, Georgios and Stanczuk, Jan and Sch{\"o}nlieb, Carola-Bibiane and Etmann, Christian},
  journal={arXiv preprint arXiv:2111.13606},
  year={2021}
}

@article{serra2022universal,
  title={Universal speech enhancement with score-based diffusion},
  author={Serr{\`a}, Joan and Pascual, Santiago and Pons, Jordi and Araz, R Oguz and Scaini, Davide},
  journal={arXiv preprint arXiv:2206.03065},
  year={2022}
}

@inproceedings{pascual2023full,
  title={Full-band general audio synthesis with score-based diffusion},
  author={Pascual, Santiago and Bhattacharya, Gautam and Yeh, Chunghsin and Pons, Jordi and Serr{\`a}, Joan},
  booktitle={ICASSP 2023-2023 IEEE International Conference on Acoustics, Speech and Signal Processing (ICASSP)},
  pages={1--5},
  year={2023},
  organization={IEEE}
}

@inproceedings{ruiz2023dreambooth,
  title={Dreambooth: Fine tuning text-to-image diffusion models for subject-driven generation},
  author={Ruiz, Nataniel and Li, Yuanzhen and Jampani, Varun and Pritch, Yael and Rubinstein, Michael and Aberman, Kfir},
  booktitle={Proceedings of the IEEE/CVF Conference on Computer Vision and Pattern Recognition},
  pages={22500--22510},
  year={2023}
}

@inproceedings{popov2021grad,
  title={Grad-tts: A diffusion probabilistic model for text-to-speech},
  author={Popov, Vadim and Vovk, Ivan and Gogoryan, Vladimir and Sadekova, Tasnima and Kudinov, Mikhail},
  booktitle={International Conference on Machine Learning},
  pages={8599--8608},
  year={2021},
  organization={PMLR}
}

@inproceedings{kim2022diffusionclip,
  title={Diffusionclip: Text-guided diffusion models for robust image manipulation},
  author={Kim, Gwanghyun and Kwon, Taesung and Ye, Jong Chul},
  booktitle={Proceedings of the IEEE/CVF Conference on Computer Vision and Pattern Recognition},
  pages={2426--2435},
  year={2022}
}

@inproceedings{richter2023speech,
  title={Speech Signal Improvement Using Causal Generative Diffusion Models},
  author={Richter, Julius and Welker, Simon and Lemercier, Jean-Marie and Lay, Bunlong and Peer, Tal and Gerkmann, Timo},
  booktitle={ICASSP 2023-2023 IEEE International Conference on Acoustics, Speech and Signal Processing (ICASSP)},
  pages={1--2},
  year={2023},
  organization={IEEE}
}

@article{wu2023duplex,
  title={Duplex Diffusion Models Improve Speech-to-Speech Translation},
  author={Wu, Xianchao},
  journal={arXiv preprint arXiv:2305.12628},
  year={2023}
}

@article{qiang2023minimally,
  title={Minimally-Supervised Speech Synthesis with Conditional Diffusion Model and Language Model: A Comparative Study of Semantic Coding},
  author={Qiang, Chunyu and Li, Hao and Ni, Hao and Qu, He and Fu, Ruibo and Wang, Tao and Wang, Longbiao and Dang, Jianwu},
  journal={arXiv preprint arXiv:2307.15484},
  year={2023}
}

@article{fan2023generative,
  title={Generative diffusion models on graphs: Methods and applications},
  author={Fan, Wenqi and Liu, Chengyi and Liu, Yunqing and Li, Jiatong and Li, Hang and Liu, Hui and Tang, Jiliang and Li, Qing},
  journal={arXiv preprint arXiv:2302.02591},
  year={2023}
}

@article{zhang2023survey,
  title={A survey on graph diffusion models: Generative ai in science for molecule, protein and material},
  author={Zhang, Mengchun and Qamar, Maryam and Kang, Taegoo and Jung, Yuna and Zhang, Chenshuang and Bae, Sung-Ho and Zhang, Chaoning},
  journal={arXiv preprint arXiv:2304.01565},
  year={2023}
}

@article{hyvarinen2005estimation,
  title={Estimation of non-normalized statistical models by score matching.},
  author={Hyv{\"a}rinen, Aapo and Dayan, Peter},
  journal={Journal of Machine Learning Research},
  volume={6},
  number={4},
  year={2005}
}

@article{vincent2011connection,
  title={A connection between score matching and denoising autoencoders},
  author={Vincent, Pascal},
  journal={Neural computation},
  volume={23},
  number={7},
  pages={1661--1674},
  year={2011},
  publisher={MIT Press}
}

@article{song2019generative,
  title={Generative modeling by estimating gradients of the data distribution},
  author={Song, Yang and Ermon, Stefano},
  journal={Advances in neural information processing systems},
  volume={32},
  year={2019}
}

@article{ho2020denoising,
  title={Denoising diffusion probabilistic models},
  author={Ho, Jonathan and Jain, Ajay and Abbeel, Pieter},
  journal={Advances in neural information processing systems},
  volume={33},
  pages={6840--6851},
  year={2020}
}

@article{song2020denoising,
  title={Denoising diffusion implicit models},
  author={Song, Jiaming and Meng, Chenlin and Ermon, Stefano},
  journal={arXiv preprint arXiv:2010.02502},
  year={2020}
}

@article{ramesh2022hierarchical,
  title={Hierarchical text-conditional image generation with clip latents},
  author={Ramesh, Aditya and Dhariwal, Prafulla and Nichol, Alex and Chu, Casey and Chen, Mark},
  journal={arXiv preprint arXiv:2204.06125},
  year={2022}
}

@article{song2020score,
  title={Score-based generative modeling through stochastic differential equations},
  author={Song, Yang and Sohl-Dickstein, Jascha and Kingma, Diederik P and Kumar, Abhishek and Ermon, Stefano and Poole, Ben},
  journal={arXiv preprint arXiv:2011.13456},
  year={2020}
}

@article{de2021diffusion,
  title={Diffusion Schr{\"o}dinger bridge with applications to score-based generative modeling},
  author={De Bortoli, Valentin and Thornton, James and Heng, Jeremy and Doucet, Arnaud},
  journal={Advances in Neural Information Processing Systems},
  volume={34},
  pages={17695--17709},
  year={2021}
}

@article{bretagnolle1979estimation,
  title={Estimation des densit{\'e}s: risque minimax},
  author={Bretagnolle, Jean and Huber, Catherine},
  journal={Zeitschrift f{\"u}r Wahrscheinlichkeitstheorie und verwandte Gebiete},
  volume={47},
  pages={119--137},
  year={1979},
  publisher={Springer}
}

@inproceedings{cranko2019boosted,
  title={Boosted density estimation remastered},
  author={Cranko, Zac and Nock, Richard},
  booktitle={International Conference on Machine Learning},
  pages={1416--1425},
  year={2019},
  organization={PMLR}
}

@inproceedings{husain2020local,
  title={Local differential privacy for sampling},
  author={Husain, Hisham and Balle, Borja and Cranko, Zac and Nock, Richard},
  booktitle={International Conference on Artificial Intelligence and Statistics},
  pages={3404--3413},
  year={2020},
  organization={PMLR}
}

@article{soen2020data,
  title={Data preprocessing to mitigate bias with boosted fair mollifiers},
  author={Soen, Alexander and Husain, Hisham and Nock, Richard},
  journal={arXiv preprint arXiv:2012.00188},
  year={2020}
}

@inproceedings{husain2022adversarial,
  title={Adversarial interpretation of Bayesian inference},
  author={Husain, Hisham and Knoblauch, Jeremias},
  booktitle={International Conference on Algorithmic Learning Theory},
  pages={553--572},
  year={2022},
  organization={PMLR}
}

@book{bakry2014analysis,
  title={Analysis and geometry of Markov diffusion operators},
  author={Bakry, Dominique and Gentil, Ivan and Ledoux, Michel and others},
  volume={103},
  year={2014},
  publisher={Springer}
}

@article{knoblauch2019generalized,
  title={Generalized variational inference: Three arguments for deriving new posteriors},
  author={Knoblauch, Jeremias and Jewson, Jack and Damoulas, Theodoros},
  journal={arXiv preprint arXiv:1904.02063},
  year={2019}
}

@article{watson2023novo,
  title={De novo design of protein structure and function with RFdiffusion},
  author={Watson, Joseph L and Juergens, David and Bennett, Nathaniel R and Trippe, Brian L and Yim, Jason and Eisenach, Helen E and Ahern, Woody and Borst, Andrew J and Ragotte, Robert J and Milles, Lukas F and others},
  journal={Nature},
  volume={620},
  number={7976},
  pages={1089--1100},
  year={2023},
  publisher={Nature Publishing Group UK London}
}

@article{geffner2025proteina,
  title={La-proteina: Atomistic protein generation via partially latent flow matching},
  author={Geffner, Tomas and Didi, Kieran and Cao, Zhonglin and Reidenbach, Danny and Zhang, Zuobai and Dallago, Christian and Kucukbenli, Emine and Kreis, Karsten and Vahdat, Arash},
  journal={arXiv preprint arXiv:2507.09466},
  year={2025}
}

@article{abramson2024accurate,
  title={Accurate structure prediction of biomolecular interactions with AlphaFold 3},
  author={Abramson, Josh and Adler, Jonas and Dunger, Jack and Evans, Richard and Green, Tim and Pritzel, Alexander and Ronneberger, Olaf and Willmore, Lindsay and Ballard, Andrew J and Bambrick, Joshua and others},
  journal={Nature},
  volume={630},
  number={8016},
  pages={493--500},
  year={2024},
  publisher={Nature Publishing Group UK London}
}

@article{dhariwal2021diffusion,
  title={Diffusion models beat gans on image synthesis},
  author={Dhariwal, Prafulla and Nichol, Alexander},
  journal={Advances in neural information processing systems},
  volume={34},
  pages={8780--8794},
  year={2021}
}

@article{li2024d,
  title={O (d/T) convergence theory for diffusion probabilistic models under minimal assumptions},
  author={Li, Gen and Yan, Yuling},
  journal={arXiv preprint arXiv:2409.18959},
  year={2024}
}

@article{conforti2025diffusion,
    author = {Conforti, Giovanni and Durmus, Alain and Silveri, Marta Gentiloni},
title = {{KL} Convergence Guarantees for Score Diffusion Models under Minimal Data Assumptions},
journal = {SIAM Journal on Mathematics of Data Science},
volume = {7},
number = {1},
pages = {86-109},
year = {2025},
doi = {10.1137/23M1613670},
}

@misc{benton2024nearlydlinearconvergencebounds,
      title={Nearly $d$-Linear Convergence Bounds for Diffusion Models via Stochastic Localization}, 
      author={Joe Benton and Valentin De Bortoli and Arnaud Doucet and George Deligiannidis},
      year={2024},
      eprint={2308.03686},
      archivePrefix={arXiv},
      primaryClass={stat.ML},
      url={https://arxiv.org/abs/2308.03686}, 
}

@ARTICLE{Li2024-em,
  title         = "{O}(d/{T}) convergence theory for diffusion probabilistic
                   models under minimal assumptions",
  author        = "Li, Gen and Yan, Yuling",
  journal       = "arXiv [cs.LG]",
  month         =  "27~" # sep,
  year          =  2024,
  archivePrefix = "arXiv",
  primaryClass  = "cs.LG"
}

@ARTICLE{Li2024-rn,
  title         = "A sharp convergence theory for the probability flow {ODEs} of
                   diffusion models",
  author        = "Li, Gen and Wei, Yuting and Chi, Yuejie and Chen, Yuxin",
  journal       = "arXiv [cs.LG]",
  month         =  "5~" # aug,
  year          =  2024,
  archivePrefix = "arXiv",
  primaryClass  = "cs.LG"
}

@ARTICLE{Jiao2024-si,
  title         = "Instance-dependent convergence theory for diffusion models",
  author        = "Jiao, Yuchen and Li, Gen",
  journal       = "arXiv [stat.ML]",
  month         =  "17~" # oct,
  year          =  2024,
  archivePrefix = "arXiv",
  primaryClass  = "stat.ML"
}

@article{thornton2025composition,
  title={Composition and control with distilled energy diffusion models and sequential monte carlo},
  author={Thornton, James and B{\'e}thune, Louis and Zhang, Ruixiang and Bradley, Arwen and Nakkiran, Preetum and Zhai, Shuangfei},
  journal={arXiv preprint arXiv:2502.12786},
  year={2025}
}

@article{skreta2025feynman,
  title={Feynman-kac correctors in diffusion: Annealing, guidance, and product of experts},
  author={Skreta, Marta and Akhound-Sadegh, Tara and Ohanesian, Viktor and Bondesan, Roberto and Aspuru-Guzik, Al{\'a}n and Doucet, Arnaud and Brekelmans, Rob and Tong, Alexander and Neklyudov, Kirill},
  journal={arXiv preprint arXiv:2503.02819},
  year={2025}
}

@article{singhal2025general,
  title={A general framework for inference-time scaling and steering of diffusion models},
  author={Singhal, Raghav and Horvitz, Zachary and Teehan, Ryan and Ren, Mengye and Yu, Zhou and McKeown, Kathleen and Ranganath, Rajesh},
  journal={arXiv preprint arXiv:2501.06848},
  year={2025}
}
